\pdfoutput=1
\documentclass{article}

\usepackage[utf8]{inputenc} 
\usepackage{booktabs}       
\usepackage{amsfonts}       
\usepackage{nicefrac}       
\usepackage{microtype}      

\usepackage{caption}
\usepackage{subcaption}
\usepackage[table,xcdraw,dvipsnames]{xcolor}
\usepackage{amssymb}
\usepackage{amsmath}
\usepackage{amsthm}
\usepackage{mathtools}
\usepackage{multirow}
\usepackage[hypertexnames=false]{hyperref}       
\usepackage{float}

\usepackage[accepted]{icml2021}

\icmltitlerunning{Feature Importance Explanations for Temporal Black-Box Models}

\begin{document}
	
\twocolumn[
\icmltitle{Feature Importance Explanations for Temporal Black-Box Models}




\begin{icmlauthorlist}
	\icmlauthor{Akshay Sood}{wisc}
	\icmlauthor{Mark Craven}{wisc}
\end{icmlauthorlist}

\icmlaffiliation{wisc}{Department of Computer Sciences, Department of 
Biostatistics \& Medical Informatics, University of Wisconsin-Madison}

\icmlcorrespondingauthor{Akshay Sood}{sood@cs.wisc.edu}


\vskip 0.3in
]



\printAffiliationsAndNotice{}  

\newtheorem{thm}{Theorem}[section]
\newtheorem{prop}[thm]{Proposition}
\newtheorem*{cor}{Corollary}
\newcommand{\X}{X_1}
\newcommand{\Xk}{X_{1, k}}
\newcommand{\Xkprime}{X_{1, k'}}
\newcommand{\bX}{\mathbf{X_1}}
\newcommand{\x}{x_1}
\newcommand{\Y}{X_2}
\newcommand{\bY}{\mathbf{X_2}}
\newcommand{\y}{x_2}
\newcommand{\Z}{Y}
\newcommand{\z}{y}
\newcommand{\g}{g_1}
\newcommand{\h}{g_2}
\newcommand{\E}{\mathbb{E}}
\newcommand{\cov}{\mathrm{cov}}
\newcommand{\W}{W}
\newcommand{\Wbar}{\overline{W}}
\newcommand{\Wstar}{W^*}
\newcommand{\Wstarbar}{\Wbar^*}

\begin{abstract}
Models in the supervised learning framework may capture rich and complex 
representations over the features that are hard for humans to interpret. 
Existing methods to 
explain such models are often specific to architectures and data where
the features do not have a time-varying component. In 
this work, we propose TIME, a method to explain models that are inherently 
temporal 
in nature. Our approach (i) uses a model-agnostic permutation-based approach to 
analyze global feature importance, (ii) identifies the importance of salient 
features 
with respect to their temporal
ordering as well as localized windows of influence, and (iii) uses hypothesis 
testing to provide statistical rigor.

\end{abstract}

\section{Introduction}

The last decade has seen an explosion in models that 
learn rich representations over large, complex parameter spaces. These have 
increasingly been applied in domains with a high degree of social impact, such 
as healthcare, but this very complexity makes them black-boxes whose 
decision-making is hard to 
explain, a critical deficit in many such domains. There has thus been a 
concomitant rise in methods to generate explanations for black-box models.
Existing research has largely focused on explaining models trained over tabular 
data, 
where each feature takes a single value per instance, instead of explaining 
temporal models,
where the instances consist of sequences or time series. 

Most explanation methods, including model-agnostic feature importance methods 
such as LIME~\cite{ribeiroWhyShouldTrust2016} and 
SHAP~\cite{lundbergUnifiedApproachInterpreting2017}, are designed for tabular 
representations.~\citet{ismailBenchmarkingDeepLearning2020} demonstrate the 
unreliability and inaccuracy of several explanation methods designed for 
tabular data in identifying feature importance in temporal models. Some 
approaches have focused on interpreting specific temporal model architectures, 
such as 
recurrent neural networks~\cite{karpathyVisualizingUnderstandingRecurrent2015, 
sureshClinicalInterventionPrediction2017, ismailInputCellAttentionReduces2019} 
and attention-based 
models~\cite{choiRETAINInterpretablePredictive2016, 
zhangATTAINAttentionbasedTimeAware2019}, while others have explored methods to 
encourage temporal models during training to be more interpretable using
tree regularization~\cite{wuSparsityTreeRegularization2017} and game-theoretic 
characterizations~\cite{leeGameTheoreticInterpretabilityTemporal2018}. However, 
model-agnostic explanation for temporal models has begun to be addressed only 
recently.~\citet{tonekaboniWhatWentWrong2020} propose FIT, a method to assign 
importance scores to observations over time for sequence-sequence 
models, and ~\citet{bentoTimeSHAPExplainingRecurrent2020} propose TimeSHAP, an 
extension of 
SHAP to temporal models. Importantly, these methods focus on local 
interpretability, which seeks to explain individual predictions in terms of 
their 
important features, rather than global interpretability, which seeks to 
identify features important to the model as a 
whole~\cite{ibrahimGlobalExplanationsNeural2019, 
covertUnderstandingGlobalFeature2020}.

Our approach is most similar to permutation-based 
methods that identify global feature importance for learned models.
\citet{breimanRandomForests2001} uses permutations to identify 
important features in random forests, and many
variants of feature importance using permutation tests have since been studied 
~\cite{stroblConditionalVariableImportance2008,
	ojalaPermutationTestsStudying2010, 
	altmannPermutationImportanceCorrected2010, 
	gregoruttiGroupedVariableImportance2015, 
	fisherAllModelsAre2019, zhouUnbiasedMeasurementFeature2020}.
The simplicity and 
generality of permutation tests makes them attractive as a tool for
model-agnostic explanation.
However, existing methods have focused on permutations of individual features, 
or in some cases groups of features, as part of a tabular representation.
We extend permutation 
tests to temporal models by using the temporal structure of the data to 
guide the selection of features to be permuted and the resulting explanations.

\begin{figure*}[t]
	\centering
	\begin{subfigure}[t]{\linewidth}
		\includegraphics[width=\textwidth]{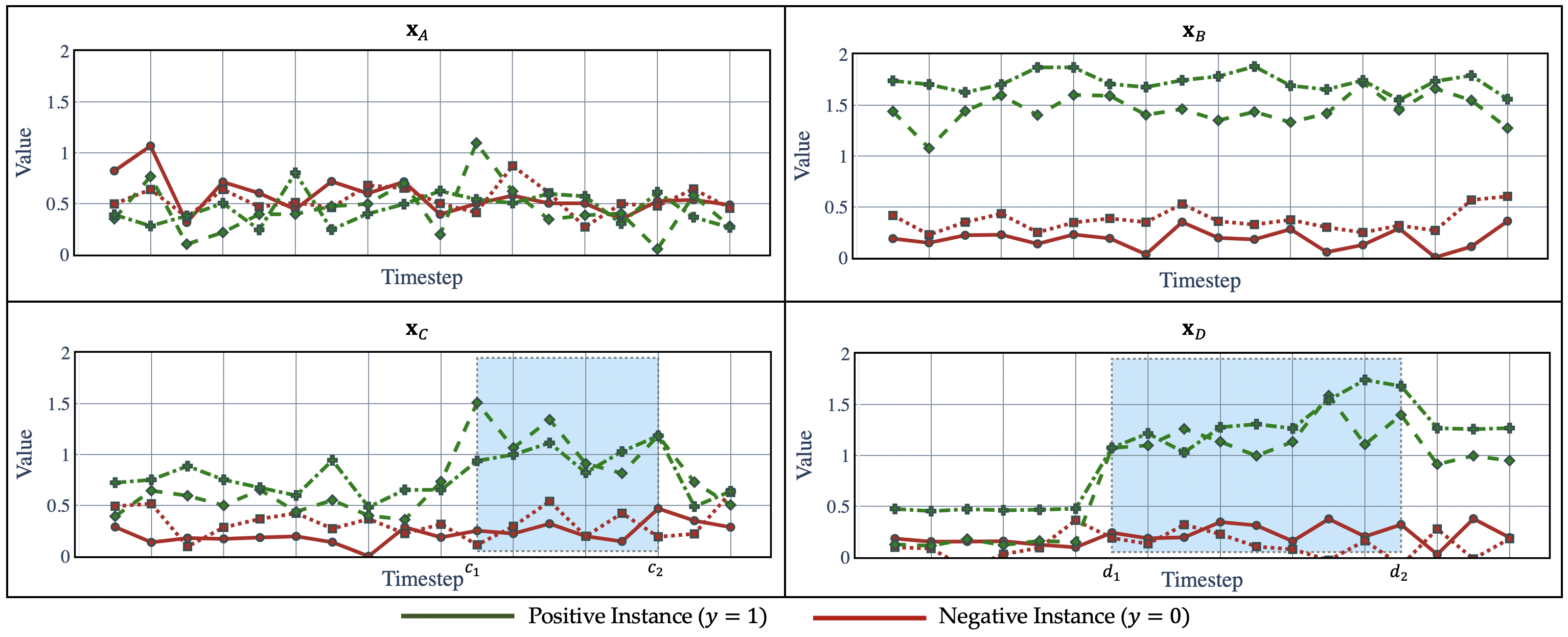}
		\caption{}
		\label{fig:series}
	\end{subfigure}
	\begin{subfigure}[t]{\linewidth}
		\includegraphics[width=\textwidth]{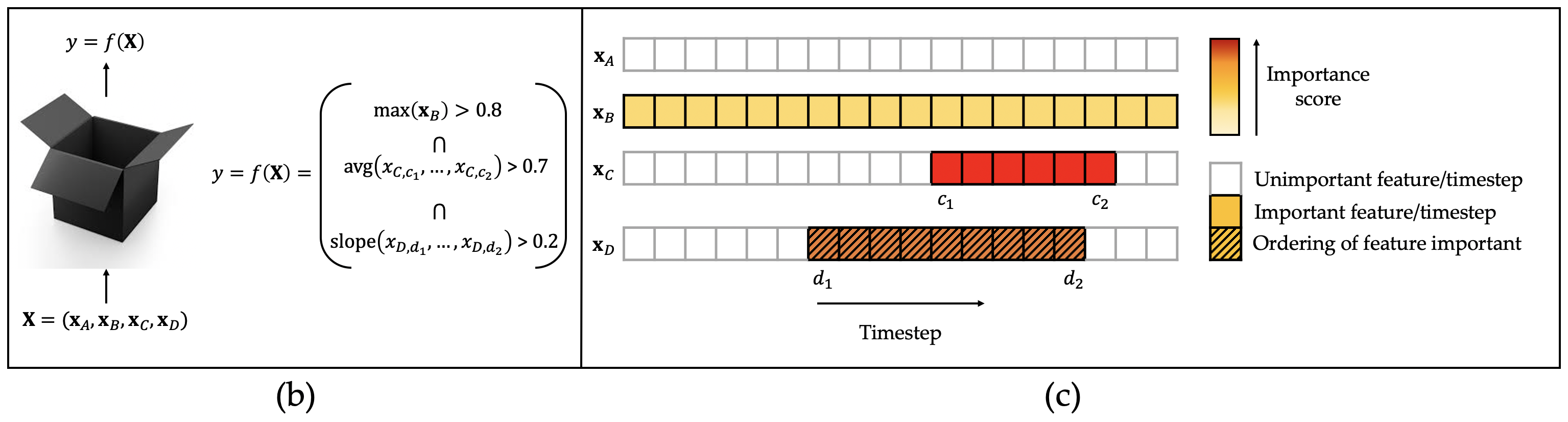}
		\label{fig:toymodel}
	\end{subfigure}
	\caption{(a) Time series for positive (green) and negative (red) instances 
	for four different 
		features, illustrating temporal properties of the features that a 
		learned 
		model may capture.
		(b) A trained binary classification model 
		over the four time-varying features, whose 
		underlying function uses the features' temporal properties to capture 
		the 
		target concept. $\mathbf{x}_A$ is 
		not used by the model; all timesteps for $\mathbf{x}_B$ are equally 
		important; the model focuses on windows $[c_1, c_2]$ and 
		$[d_1, d_2]$ for $\mathbf{x}_C$ and $\mathbf{x}_D$ 
		respectively; the ordering of values is 
		important only for $\mathbf{x}_D$. (c) The output of TIME, 
		showing 
		for each feature (i) its overall importance to the model, (ii) the most 
		important window that the model focuses on, and (iii) whether the 
		ordering 
		of the values within the window is important to the model.}
	\label{fig:illustration}
\end{figure*}

In this work, we propose Temporal Importance Model Explanation (TIME), a 
method for explaining temporal black-box models.
Our approach is model-agnostic, produces global explanations, and elicits 
specific properties of temporal models.
It takes as input a learned model over features representing 
sequences or time-series, and a test data set used to evaluate the model, 
and does the following: (i) it identifies  
features that are important for the model's predictions across the distribution 
of 
instances, (ii) for each such feature, it identifies the most important 
temporal window that the model focuses on, (iii) it determines whether the 
model's predictions are dependent on the ordering of the values within the 
window, (iv) it uses hypothesis testing and a false discovery 
rate control 
methodology to identify important features and their temporal properties with 
statistical rigor, and (v) it treats the model as a black-box, and thus and may 
be used to analyze a wide variety of temporal or sequential models.
Figure~\ref{fig:illustration} illustrates the setting and our approach.
\section{Methods}

\subsection{Identifying Important Features/Timesteps via Permutation}
\label{sec:feature-relevance}

\paragraph{Non-temporal models.} We first outline the case of a model trained 
on a tabular data set where each feature takes a single value per 
instance.
Consider a model $f$ over $D$ features, trained to predict a target $y$. 
We are interested in examining the importance of a given feature $j$ for the 
model in predicting $y$. To focus on the 
model's generalization performance, we assume a test set comprising $M$ 
instances is available to analyze the model.
Let $(\mathbf{x}^{(i)}, y^{(i)})$ be the $i^{th}$ instance-target pair, and 
$\mathcal{L}$ be a loss function linking the model 
output $f(\mathbf{x})$ to the target $y$.

We define the \textit{baseline} output of the model $f$ for instance $i$ as:
\begin{equation}
f \left( \mathbf{x}^{(i)} \right) = f \left( x_1^{(i)}, x_2^{(i)}, 
\ldots, 
x_j^{(i)}, \ldots, x_D^{(i)} \right).
\end{equation}
The \textit{perturbed} output of the model for instance $i$ w.r.t feature $j$ 
and another instance $l \ne i$ is defined as:
\begin{equation}
\textcolor{Black}{f \left( \textcolor{Black}{\mathbf{x}_j^{(i, l)}} \right)} 
= f \left( 
x_1^{(i)}, x_2^{(i)}, \ldots, \textcolor{BrickRed}{x_j^{(l)}}, \ldots, 
x_D^{(i)} \right)
\end{equation}
where the value of feature $j$ is replaced by its corresponding value from 
instance $l$, as shown in Figure~\ref{fig:perturbmatrix}. Then, 
we can compute the change in loss between the 
perturbed and baseline losses as:
\begin{equation}
\Delta \mathcal{L}_j^{(i, l)} = \textcolor{Black}{\mathcal{L} \left[ y^{(i)}, 
f \left( \textcolor{Black}{\mathbf{x}_j^{(i, l)}} \right) \right]} - 
\textcolor{Black}{\mathcal{L} \left[ y^{(i)}, f \left( \mathbf{x}^{(i)} 
\right) \right]}.
\label{eqn:loss-change-matrix}
\end{equation}
Let $S_i \subset \{1, 2, \ldots M\} \backslash \{i\}$ be a subset of instances 
excluding~$i$. We replace the value of feature $j$ for instance $i$ with its 
value from each instance in $S_i$ and compute the average change in loss as:
\begin{equation}
\Delta \bar{\mathcal{L}_j}^{(i)} = \frac{1}{|S_i|} \sum \limits_{l \in S_i} 
\Delta 
\mathcal{L}_j^{(i, l)}.
\label{eqn:avg-loss-change-matrix}
\end{equation}
By averaging this calculation over all instances $i = 1 \ldots M$ using subsets 
$S_i$ selected by permuting the data set,
we compute the importance score of feature $j$ as:
\begin{equation}
I \left( f, j \right) = \frac{1}{M} \sum\limits_{i=1}^{M} \Delta 
\bar{\mathcal{L}_j}^{(i)}.
\label{eqn:importance-matrix}
\end{equation}

A model includes many
features, all of which may have some effect on the model's
output, but only some
of which are useful in predicting the target. We consider a feature to be 
important if it has a positive association 
with the target, in the sense that permuting the value of the feature
increases the model's loss on average. We use hypothesis testing to capture 
this notion of importance, as outlined in 
Subsection~\ref{sec:hypothesis-testing}.

\begin{figure}[h!]
	\centering
	\begin{subfigure}[c]{0.35 \linewidth}
		\includegraphics[width=\textwidth]{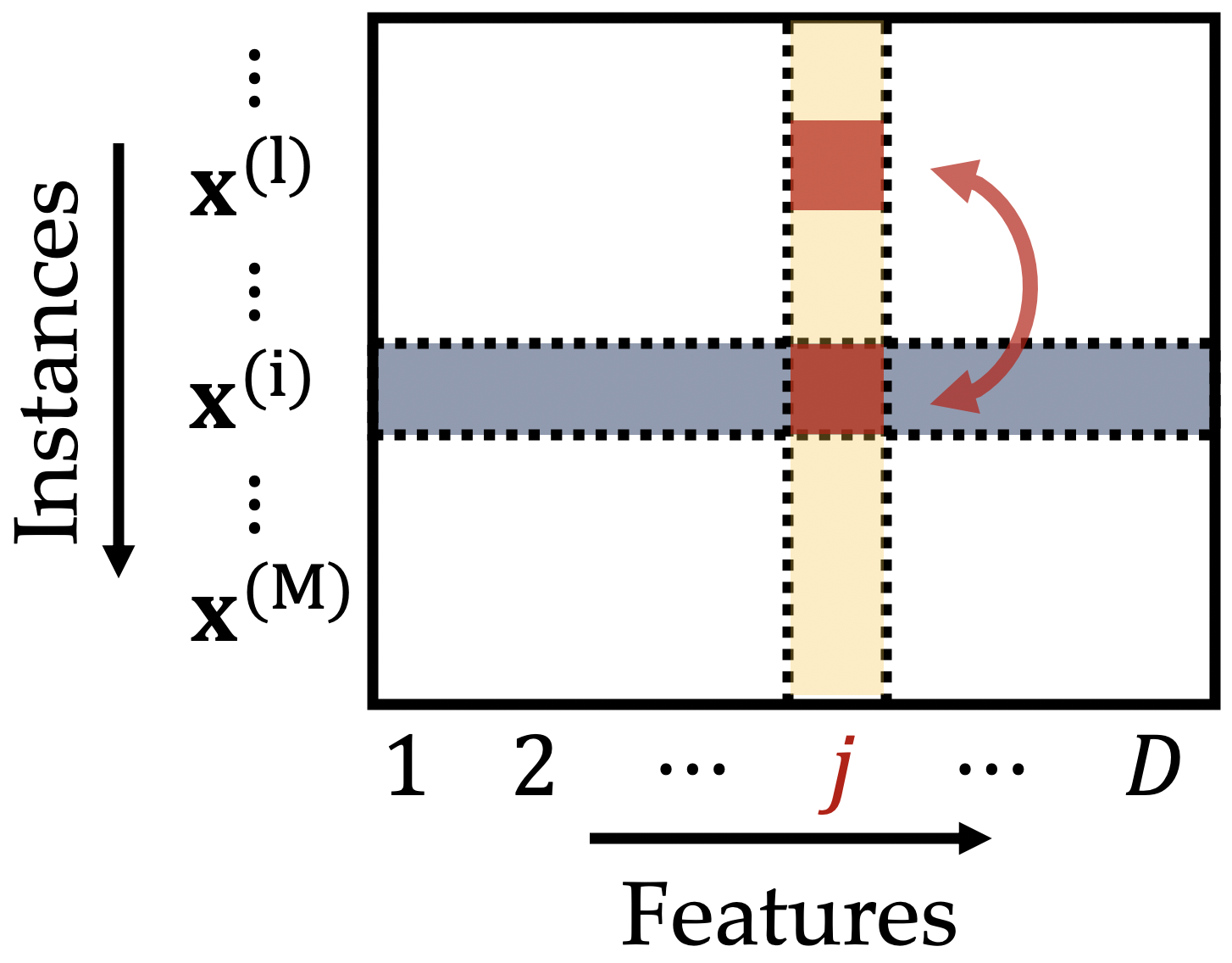}
		\caption{}
		\label{fig:perturbmatrix}
	\end{subfigure}
	\hfill
	\begin{subfigure}[c]{0.50 \linewidth}
		\includegraphics[width=\textwidth]{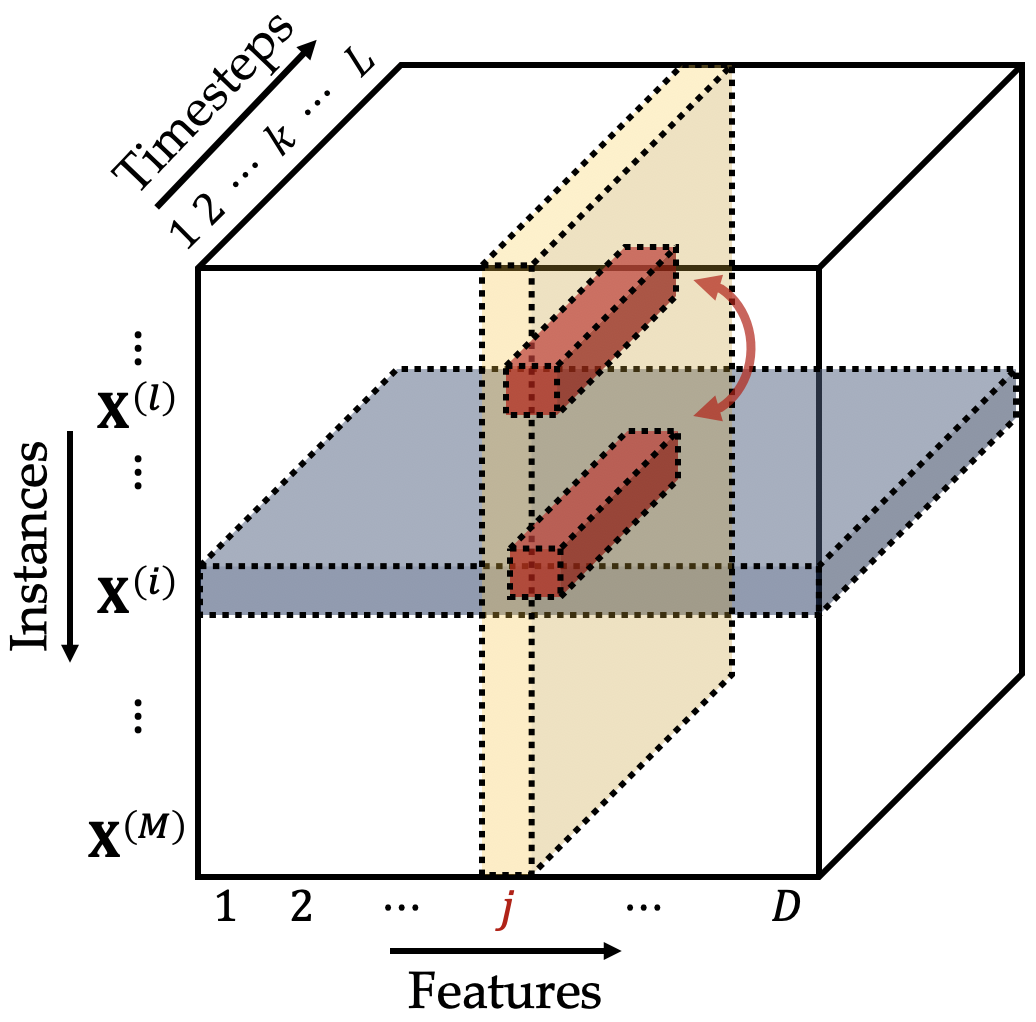}
		\caption{}
		\label{fig:perturbtensor}
	\end{subfigure} \\
	\begin{subfigure}[c]{0.50 \linewidth}
	\includegraphics[width=\textwidth]{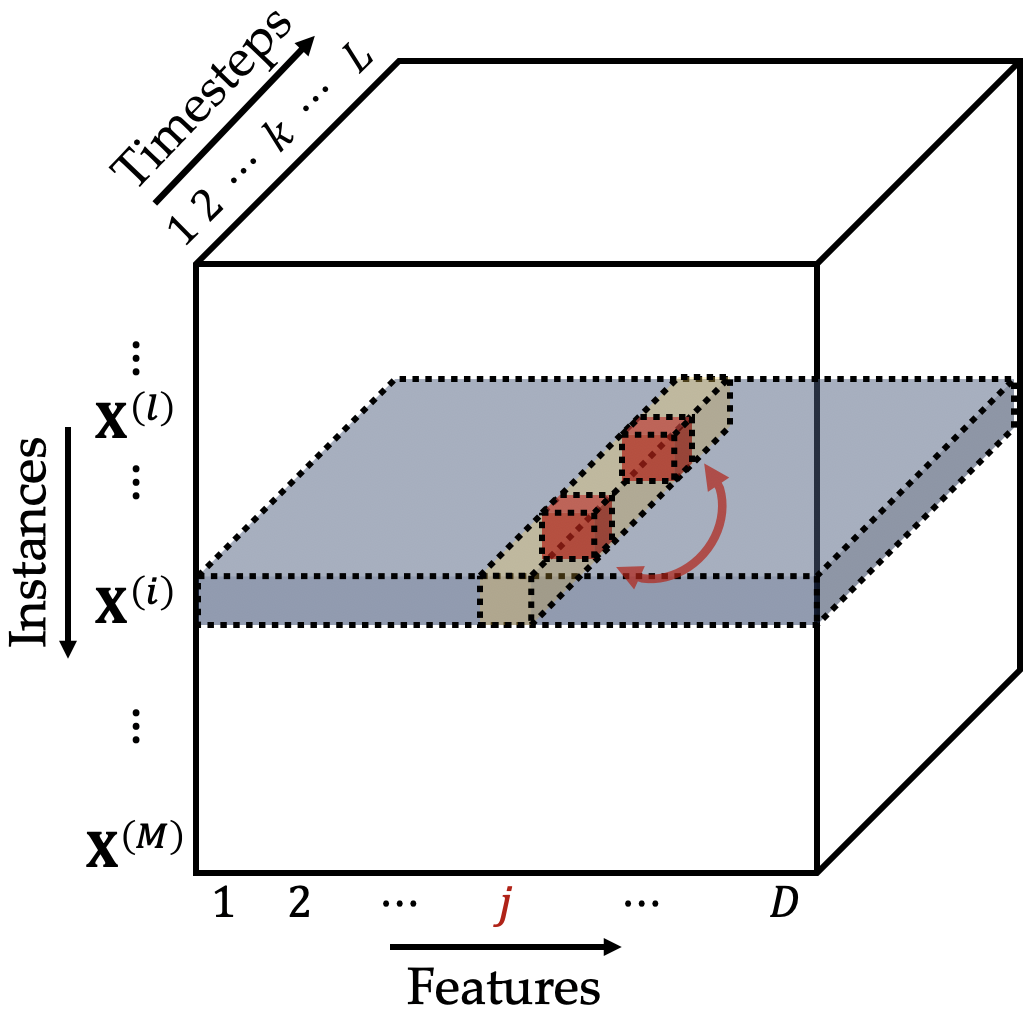}
	\caption{}
	\label{fig:perturbtime}
	\end{subfigure}
	\caption{Perturbation for instance $i$ and 
	feature $j$ to compute feature importance. (a) Data matrix showing 
	the replacement of the value of 
	feature $j$ in instance $i$ from instance $l$. (b) Data tensor showing the 
	replacement of a window of feature $j$ in instance $i$ from 
	the corresponding window of instance $l$ . (c) Time series 
	$\mathbf{x}_j^{(i)}$ 
	showing the exchange of feature values at two timesteps.}
	\label{fig:perturbation}
\end{figure}

\paragraph{Temporal models.}
We extend this idea to temporal models, where we assume that each feature is
represented by a time series of length $L$ and that the time series are aligned 
in time, so that the data is represented by 
an $M \times D \times L$ tensor, with instance $i$
represented by a matrix $\mathbf{X}^{(i)}$ and feature $j$ of instance $i$ by a 
time series $\mathbf{x}_j^{(i)} = \left< x_{j, 1}^{(i)}, x_{j, 2}^{(i)}, 
\ldots, x_{j, k}^{(i)}, \ldots x_{j, L}^{(i)} \right>$.

By unrolling in time, this may be viewed as tabular data consisting of $M$ 
instances and $D \cdot L$ features, so that permutations of individual 
features in the tabular setting correspond to permutations of 
individual timesteps in the temporal setting. However, doing so 
ignores the temporal structure of the data and correlations within time 
series. Thus, we consider joint permutations of timesteps in the temporal 
setting.

While joint permutations of subsets of features have been studied for tabular 
data~\cite{heneliusPeekBlackBox2014, 
gregoruttiGroupedVariableImportance2015}, the exponentially large number of 
subsets make the task intractable in general, and high-value subsets must be 
selected, such as through domain 
knowledge~\cite{leeUnderstandingLearnedModels2019}. Several removal-based 
explanation methods may be viewed as operating on high-value feature 
subsets~\cite{covertExplainingRemovingUnified2020}. In the temporal setting, 
the structure of the data lends itself to high-value subsets constituted by 
contiguous regions, i.e., windows, in time.

The perturbed output of the model for instance $i$ w.r.t. feature $j$ and time
window $[k_1, k_2]$ is given by:
\begin{equation}
\textcolor{Black}{f \left( \mathbf{X}_{j, [k_1, k_2]}^{(i, l)} \right)} = f 
\left( 
\mathbf{x}_1^{(i)}, 
\mathbf{x}_2^{(i)}, \ldots, \textcolor{BrickRed}{\mathbf{x}_{j, [k_1, 
k_2]}^{(i, l)}}, \ldots, \mathbf{x}_D^{(i)} \right)
\label{eqn:perturbed-model-tensor}
\end{equation}
where $\textcolor{BrickRed}{\mathbf{x}_{j, [k_1, k_2]}^{(i, l)}}$ is the time 
series for 
instance $i$ and feature $j$ with  timesteps in the window $[k_1, k_2]$ 
replaced by the corresponding window from another instance $l \ne i$, as 
shown in Figure~\ref{fig:perturbtensor}.
\begin{equation}
\textcolor{BrickRed}{\mathbf{x}_{j, [k_1, k_2]}^{(i, l)}} = \left< x_{j, 
1}^{(i)}, x_{j, 
2}^{(i)}, \ldots, \textcolor{BrickRed}{x_{j, k_1}^{(l)}, \ldots, x_{j, 
k_2}^{(l)}}, \ldots, x_{j, L}^{(i)} \right>.
\label{eqn:perturbed-feature-tensor}
\end{equation}
We compute the perturbed loss $\mathcal{L}\left[y^{(i)}, f \left( 
\mathbf{X}_{j, 
[k_1, k_2]}^{(i, l)} \right)\right]$, the change in loss 
(Equation~\ref{eqn:loss-change-matrix}), and average change in loss 
(Equation~\ref{eqn:avg-loss-change-matrix}) for instance $i$. We 
average this over all instances $i = 1 \ldots M$ to compute the 
importance score corresponding to the window $[k_1, k_2]$ for feature $j$:
\begin{equation}
I \left( f, j, [k_1, k_2] \right)
= \frac{1}{M} \sum\limits_{i=1}^{M} \Delta \bar{\mathcal{L}}_{j, [k_1, 
k_2]}^{(i)}.
\label{eqn:importance-tensor}
\end{equation}

By setting $k_1 = 1$ and $k_2 = L$, we permute the entire sequence of feature 
$j$ and get its overall importance
$I\left(f, j, [1, L] \right)$. Equation~\ref{eqn:importance-tensor}.

\subsection{Identifying Important Windows} \label{sec:important-window}
Given that the features have an explicit temporal structure, we want to 
localize the timesteps that the model may be focusing on. We assume that for a 
given feature $j$, there exists an underlying contiguous
time window $W^* = [k_1, k_2]: 1 \le k_1 < k_2 \le L$, 
so that most of the effect of perturbing $j$ lies in $W^*$.
Specifically, we consider a partitioning of the sequence into three 
windows: prior window $W_P = [1, k_1 - 1]$, important window $W^* = [k_1, 
k_2]$, and subsequent window
$W_S = [k_2 + 1, L]$ where $W_P$ are 
$W_S$ both have low importance and a size of zero or more timesteps.
In order to pin 
down the most salient timesteps, we want to find the 
smallest window $W^*$ such that
$I \left(f, j, 
W_P \right) < \left( \frac{1 - \gamma}{2} \right) I\left(f, j, 
[1, L] \right)$
and
$I \left(f, j, 
W_S \right) < \left( \frac{1 - \gamma}{2} \right) I\left(f, j, 
[1, L] \right)$, 
where the \textit{window localization parameter} $\gamma:\;0 < \gamma < 1$ 
controls the degree to which the model focuses on the 
window $W^*$.

We use a binary search algorithm to identify the largest possible prior window 
$W_P$ and subsequent window $W_S$, and by exclusion, identify the important 
window $W^*$. We start with an initial estimate $\hat{W}_P = [1, 
\hat{k}_1]$ with $\hat{k}_1 = \frac{L}{2}$. We then perturb 
$\hat{W}_P$ and observe its importance score $I \left(f, j, \hat{W}_P \right)$.

If $\hat{W}_P$ contains important timesteps, its 
importance score is likely to be 
inflated due to the breakage of correlations between all timesteps of the 
important window, 
i.e., predictors strongly associated with the 
response~\cite{nicodemusBehaviourRandomForest2010}, leading the search 
algorithm to contract $\hat{W}_P$ to exclude these 
timesteps. On the other hand, if $\hat{W}_P$ has a low importance score $I 
\left(f, j, 
\hat{W}_P \right) < \left( \frac{1 - \gamma}{2} \right) I\left(f, j, 
[1, L] \right)$, we expand it unless doing so would violate this condition.

We expand or contract $\hat{W}_P$ by updating 
$\hat{k}_1$ and repeat the perturbation until we find the 
largest $\hat{W}_P$ that satisfies $I \left(f, j, 
\hat{W}_P \right) < \left( \frac{1 - \gamma}{2} \right) I\left(f, j, 
[1, L] \right)$, and set $k_1 = |\hat{W}_P| + 1$.

We follow a similar procedure to identify $k_2$, starting from an initial 
estimate $\hat{W}_S = [\hat{k}_2, L]$ with $\hat{k}_2 = k_1 + 1$, measuring its 
importance score, and iteratively expanding or contracting it under the 
constraint 
$\hat{k}_2 > k_1$, until we identify the largest $\hat{W}_S$ that satisfies $I 
\left(f, j, \hat{W}_S \right) < 
\left( \frac{1 - \gamma}{2} \right) 
I\left(f, j, [1, L]\right)$. We select the final boundary estimates 
$k_1$ and $k_2 = L - |\hat{W}_S|$ to characterize the important window $W^*$ 
and compute its importance score.

\subsection{Determining the Importance of Feature Ordering}
\label{sec:ordering-importance}
Permutations of time series ordering have been used to detect 
circadian patterns in gene expression 
data~\cite{storchExtensiveDivergentCircadian2002, 
ptitsynCircadianClocksAre2006}.
To determine the importance of the ordering of a feature $j$ within a 
given window $[k_1, k_2]$, we permute its values within the window, as 
illustrated in Figure~\ref{fig:perturbtime}, and average across instances.
Let $\Pi_{[k_1, k_2]}$ be a 
permutation over timesteps within the window: $\Pi_{[k_1, k_2]} = 
\mathrm{Permute}\left[ \left< k_1, k_1 + 1, \ldots, k_2 \right> \right] = 
\left< {\pi}_{k_1}, {\pi}_{k_1 + 1}, \ldots, {\pi}_{k_2} \right>$. The 
perturbed model output is given by:
\begin{equation}
\textcolor{Black}{f \left( \mathbf{X}^{(i)}_{j, \Pi_{[k_1, k_2]}} \right)} = f 
\left( \mathbf{x}_1^{(i)}, \mathbf{x}_2^{(i)}, \ldots, 
\textcolor{BrickRed}{\mathbf{x}_{j, \Pi_{[k_1, k_2]}}^{(i)}}, \ldots, 
\mathbf{x}_D^{(i)} \right)
\end{equation}
where the permuted time series for instance $i$ and 
feature $j$ is given by:
\begin{equation}
\begin{split}
\textcolor{BrickRed}{\mathbf{x}_{j, \Pi_{[k_1, k_2]}}^{(i)}} =  
\langle &x_{j, 1}^{(i)}, \ldots, x_{j, k_1 - 1}^{(i)},
\textcolor{BrickRed}{x_{j, {\pi}_{k_1}}^{(i)}, \ldots, 
x_{j, {\pi}_{k_2}}^{(i)}}, \\
&x_{j, k_2 + 1}^{(i)}, \ldots x_{j, L}^{(i)} \rangle
\end{split}
\label{eqn:ordering}
\end{equation}

We compute the change in loss between the perturbed and baseline loss 
(Equation~\ref{eqn:loss-change-matrix}) for multiple permutations $\Pi_{[k_1, 
k_2]}$, and use hypothesis testing to determine the importance of the ordering 
of $j$.

\subsection{Hypothesis Testing and False Discovery Rate Control}
\label{sec:hypothesis-testing}
Several works have leveraged hypothesis
testing in conjunction with permutations to ascertain feature importance
\cite{ojalaPermutationTestsStudying2010,
	mentchQuantifyingUncertaintyRandom2016, mentchFormalHypothesisTests2017,
	leeUnderstandingLearnedModels2019, burnsInterpretingBlackBox2020}.
We use hypothesis testing to evaluate feature importance, while using the
scores from Equation~\ref{eqn:importance-matrix} to quantify the degree of
importance. Intuitively, if a feature is 
useful to the model for predicting the target, then permuting 
its values across instances should increase the model's loss on average.

Specifically, we use the formulation of permutation tests 
in~\citet{ojalaPermutationTestsStudying2010} to evaluate the effect of 
permuting a feature and ascertain its importance. We use mean loss as the test 
statistic, so that the empirical \textit{p}-value for testing the significance 
of 
permuting feature $j$ is given by:
\begin{equation}
\hat{p} = \frac{|\left\{ \Pi \in \mathcal{P}_j : \bar{\mathcal{L}}_{\Pi} \le 
\bar{\mathcal{L}} \right\}| + 1}{|\mathcal{P}_j| + 1}
\label{eqn:empirical-pvalue}
\end{equation}
where $\mathcal{P}_j$ is a set of permutations of the original data with 
feature $j$ 
permuted in some way, $\bar{\mathcal{L}}$ is the mean loss for the original 
data, and 
$\bar{\mathcal{L}}_{\Pi}$ is the mean loss for permuted data  
$\Pi$.

Depending on the choice of $\mathcal{P}_j$, we can use 
Equation~\ref{eqn:empirical-pvalue}
to test 
the overall importance, window importance, and ordering importance for a 
feature $j$. These tests may be organized as a hierarchy, as 
shown in Figure~\ref{fig:hierarchy-tests}, so that a test is performed only if 
its parent test, if any, returns a significant \textit{p}-value.

The multiplicity of hypothesis tests for a given feature and across the set of 
features leads to a multiple comparisons problem. We address this by using a
hierarchical false discovery rate (FDR) control methodology
\cite{yekutieliHierarchicalFalseDiscovery2008}, with the FDR for tests with 
the same parent controlled using the Benjamini-Hochberg 
procedure~(\citeyear{benjaminiControllingFalseDiscovery1995}). This approach 
also 
readily extends to the case when features may be arranged in a hierarchy in 
order to interpret models in terms of feature 
groups~\cite{heneliusPeekBlackBox2014, 
gregoruttiGroupedVariableImportance2015, 
leeUnderstandingLearnedModels2019}, as shown in 
Figure~\ref{fig:hierarchy-features}.

\begin{figure}[h!]
	\centering
	\begin{subfigure}{\linewidth}
			\includegraphics[width=\textwidth]{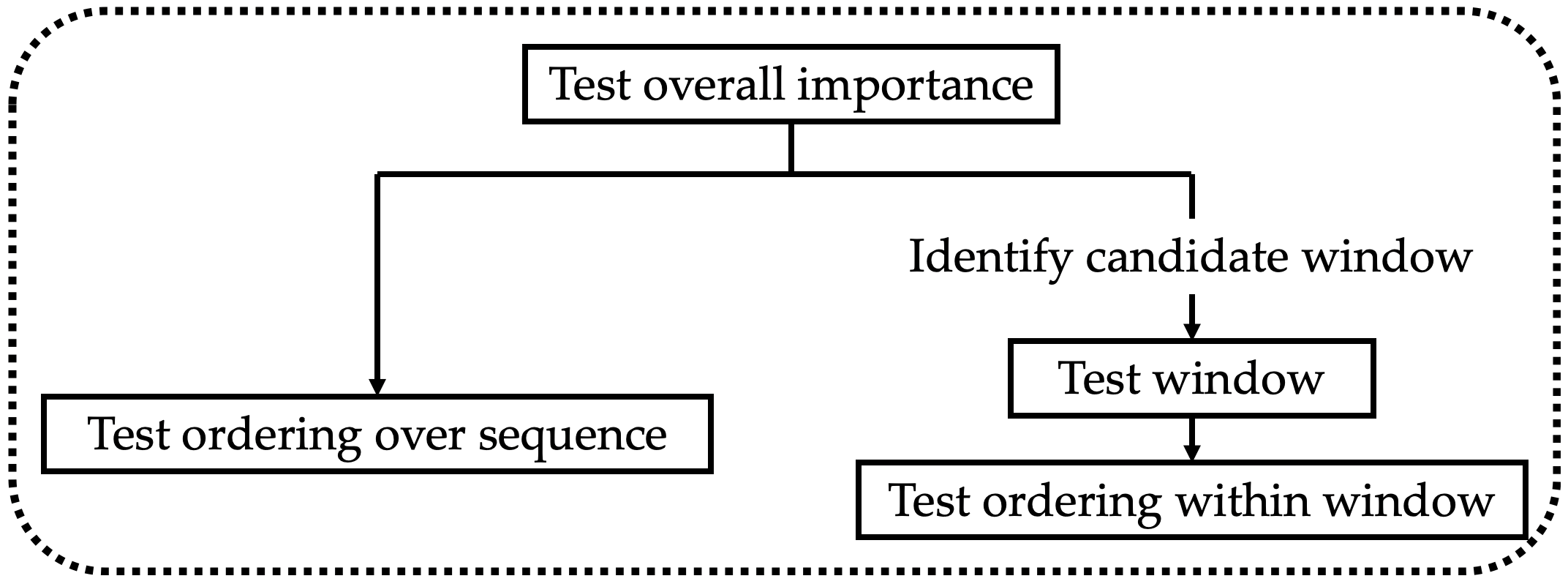}
		\caption{}
		\label{fig:hierarchy-tests}
	\end{subfigure} \\
	\begin{subfigure}{\linewidth}
	\centering
	\includegraphics[width=0.7\textwidth]{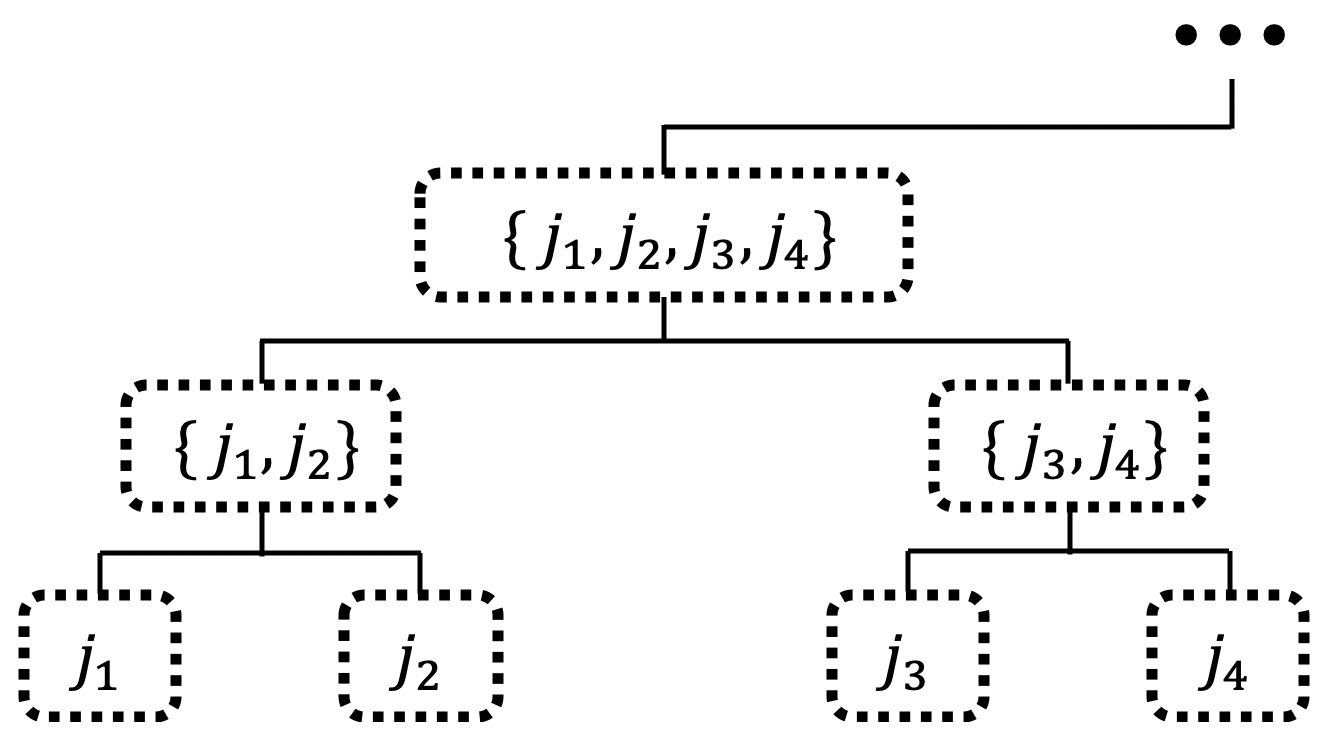}
	\caption{}
	\label{fig:hierarchy-features}
	\end{subfigure}
	\caption{(a) A hierarchy of tests used to check for a given feature its (i) 
	overall importance, (ii) important window and (iii) the importance of 
		ordering within the window. (b) A hierarchy over the 
		features, where each node is tested using the testing hierarchy shown 
		in (a). Feature groups are tested via joint permutations of their 
		constituent features. Hierarchical FDR control is used for multiple 
		testing correction, and subtrees rooted at nodes with 
		\textit{p}-values above a threshold are pruned.}
\end{figure}
\section{Results}

We evaluate TIME in two ways: by analyzing synthetic data sets and 
models where the ground truth pertaining to relevant features and their 
temporal 
properties is known, and by analyzing a long short 
term memory (LSTM) model trained to 
predict in-hospital mortality from intensive care unit (ICU) data.

\subsection{Simulations}

To evaluate our approach, we create synthetic time-series data where we 
control the generating processes for different features and the functions used 
to generate targets for the instances. We also create synthetic models that 
approximate these functions and serve as the models to be analyzed. 
We control the features that are relevant to the models, as well 
as the temporal properties of the models, including the relevant window and 
dependence on ordering for each feature. We then analyze these models using 
TIME and evaluate the results in terms of power (the fraction of relevant 
features correctly identified) and FDR (the fraction of features 
estimated to be important, but not truly relevant in the underlying 
function).

\paragraph{Synthetic data and models.}
\label{subsubec:synthetic-data-models}
We use random Markov chains to generate time series 
data with both categorical and continuous features, and randomly 
selected windows for each feature. More details of the synthetic data 
generation model are given in Section~\ref{section:supp-synthetic-data}.
For each synthetic data set, we create a 
multi-level function to generate targets for the instances. For each feature 
$j$,
we apply a \textit{feature function} $g_j$ that aggregates the values within 
the window 
$[k_1, k_2]$ of that feature:
\begin{equation}
	g_j \left( \mathbf{x}_j \right) =  \hat{g}_j \circ 
	\tilde{g}_j \circ \bar{g}_j \left( x_{j, k_1} \ldots x_{j, k_2} \right)
	\label{eqn:feature-function}
\end{equation}

where (i) $\bar{g}_j$ is aggregation operator, 
randomly selected from one of \texttt{max}, 
\texttt{average} (both insensitive to temporal ordering), 
\texttt{monotonic-weighted-average} and 
\texttt{random-weighted-average} (both sensitive to temporal ordering) 
functions, (ii) $\tilde{g}_j$ is 
randomly selected from one of \texttt{identity}, \texttt{absolute-value} 
and \texttt{square} functions and serves to potentially induce
interactions between the timesteps in the window, and (iii) $\hat{g}_j$ is a 
standardization operator that yields zero mean and unit variance for feature 
$j$ across the data set.

We designate a 
subset of all features as relevant and
take a linear combination of their feature functions to generate the target:
\begin{equation}
y = \sum\limits_{j \in \mathcal{R}} \alpha_j g_j \left( \mathbf{x}_j \right)
\label{eqn:ground-truth-model}
\end{equation}
where $\mathcal{R}$ is the set of relevant features, and coefficients 
$\alpha_j$ are sampled uniformly at random between -1 and 1.
This serves to generate responses for a regression task. To emulate a 
classification task, we choose a threshold such that half the 
instances are labeled negative and the other half are labeled positive.

The synthetic model represents an approximation of this function and is 
generated by adding a weighted linear combination of the set of irrelevant 
features $\mathcal{R'}$ to the function:
\begin{equation}
f \left( \mathbf{X} \right) = \sum\limits_{j \in \mathcal{R}} \alpha_j
g_j \left( \mathbf{x}_j \right) + \beta \left[ \sum\limits_{j' \in 
\mathcal{R'}} 
\alpha_{j'} g_{j'} \left( \mathbf{x}_{j'} \right) \right].
\label{eqn:synth-model}
\end{equation}

The terms corresponding to the irrelevant features represent
noise in the model, with the overall level of noise controlled by the 
multiplier $\beta$. The 
rationale behind the approximation is to have a realistic 
model that does not perfectly match the underlying function and whose output 
changes in a small way when irrelevant features are perturbed, but not in a way 
that consistently affects the loss function $\mathcal{L}$.

Unlike a real model where training may involve optimizing over a loss 
function, here we use a loss function only to measure the fidelity of the model 
output $f$ to the target $y$ and compute the importance of each feature. 
We use quadratic loss for regression models and 
binary cross-entropy for classification models.

\paragraph{Baseline comparisons.}
\begin{figure*}[t]
	\centering
	\begin{subfigure}[c]{0.32 \linewidth}
		\includegraphics[width=\textwidth]{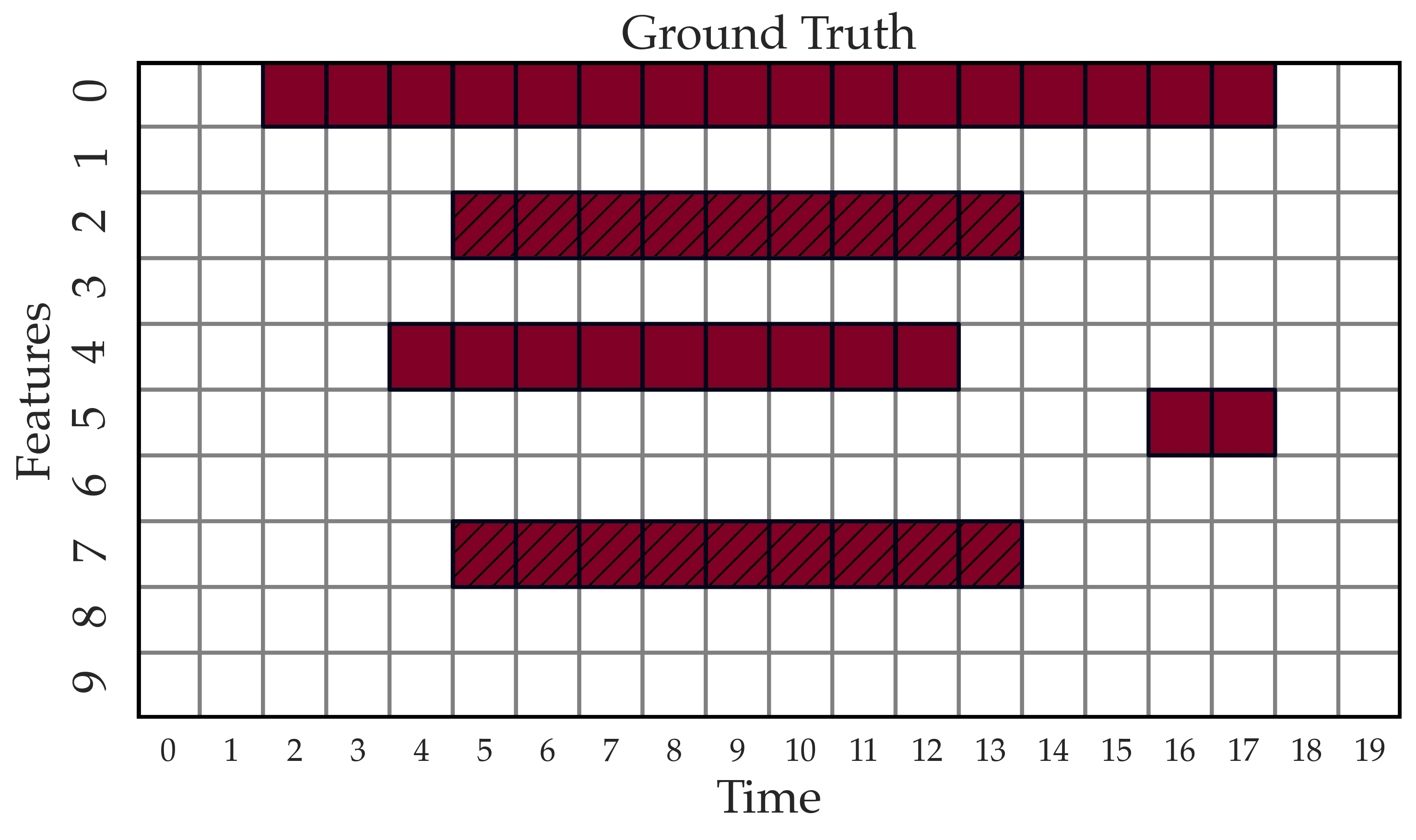}
		\caption{}
		\label{fig:baseline-ground-truth}
	\end{subfigure}
	\hfill
	\begin{subfigure}[c]{0.32 \linewidth}
		\includegraphics[width=\textwidth]{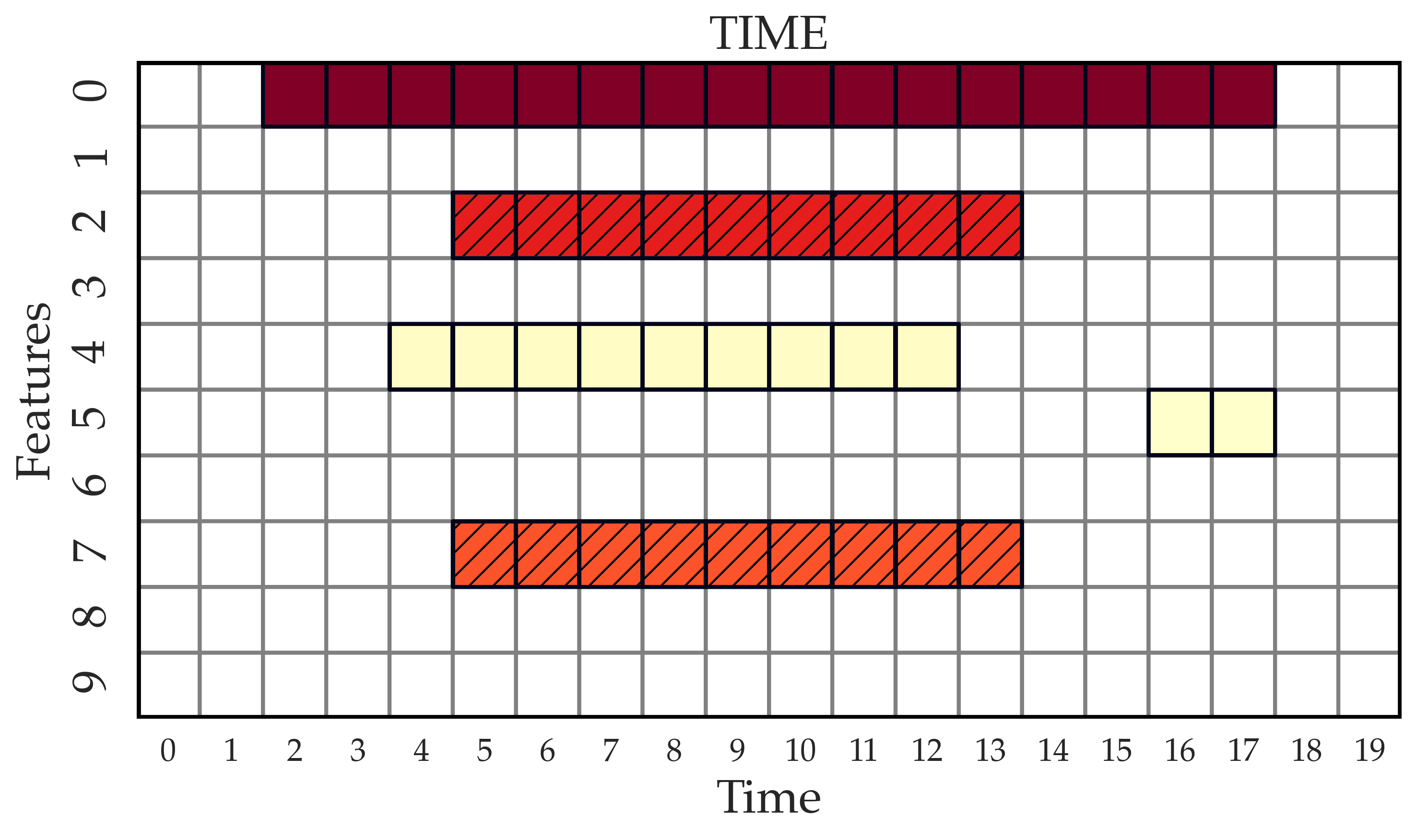}
		\caption{}
		\label{fig:baseline-our-method}
	\end{subfigure}
	\hfill
	\begin{subfigure}[c]{0.32 \linewidth}
		\includegraphics[width=\textwidth]{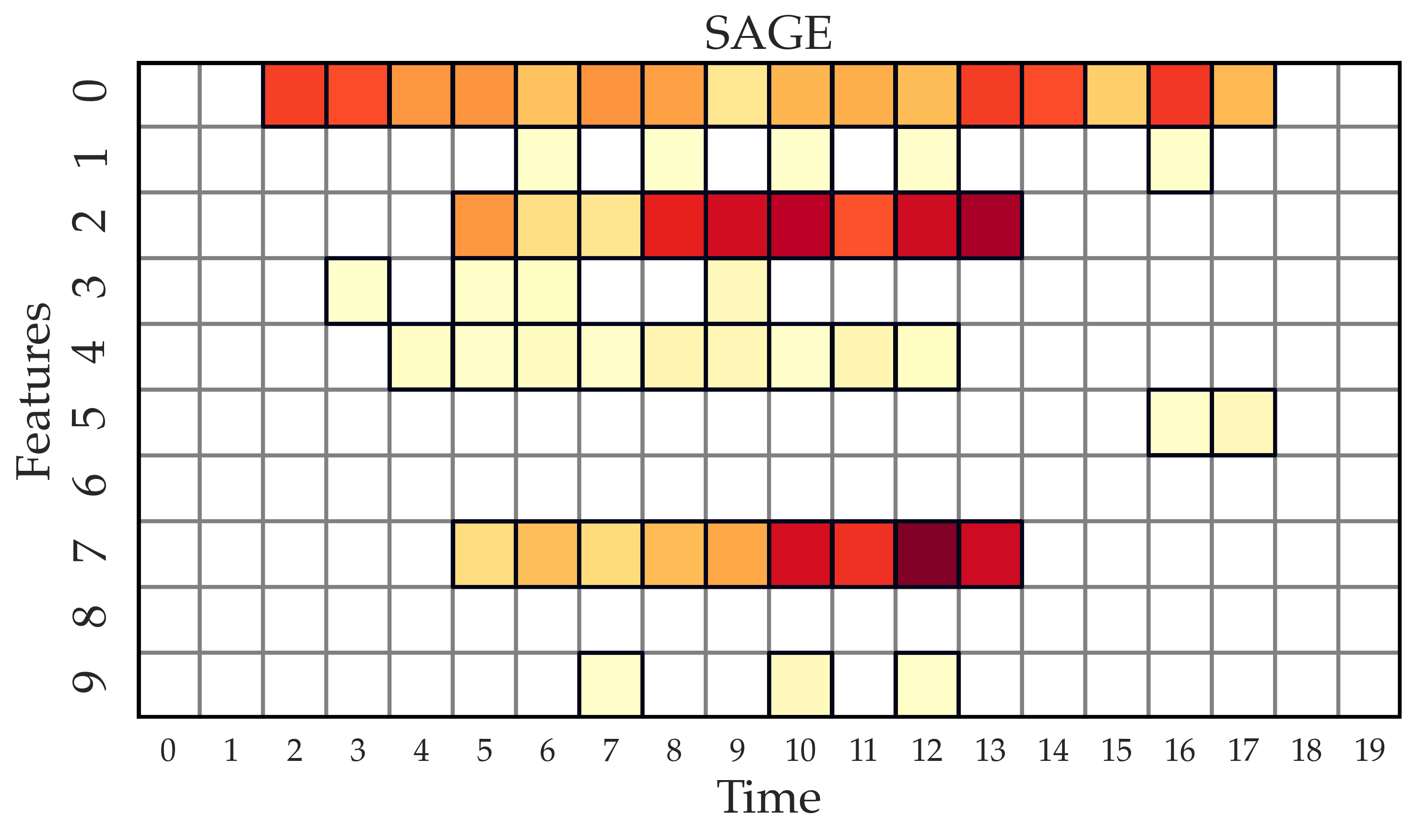}
		\caption{}
		\label{fig:baseline-sage}
	\end{subfigure}
	\caption{(a) Heat map showing relevant features and windows for the ground 
	truth model in color, with hatched textures used to indicate features 
	sensitive to 
	ordering. (b) Heat map showing importance scores for TIME, indicating 
	important features, windows and ordering. (c) Heat map showing importance 
	scores for SAGE. Darker shades indicate higher 
	importance scores for the explanation methods.}
	\label{fig:heatmaps}
\end{figure*}
We compare TIME against three baseline methods:

LIME~\cite{ribeiroWhyShouldTrust2016}: a model-agnostic 
method for local explanations. We aggregate local feature importance scores to 
generate global ones, based on the \textit{submodular pick} algorithm described 
by the authors.
We include LIME due to its widespread usage as an explanation method, and as a 
representative of other methods that focus on the model 
output rather than loss and generate local explanations.

SAGE~\cite{covertUnderstandingGlobalFeature2020}: a model-agnostic 
method that generalizes SHAP~\cite{lundbergUnifiedApproachInterpreting2017} to 
global explanations. SAGE is intractable to compute exactly, so we use two 
approximations provided by the authors: sampling held-out features from 
(i) their marginal distributions, or (ii) from 
reference values (for which we use mean values); we refer to these as SAGE 
and SAGE-m respectively.

PERM: a method using traditional permutation tests on individual, i.e., tabular 
features, 
rather than sequences. This is also a global 
feature importance method, although it does not take into account the temporal 
structure of the data.

Since the baseline methods are designed for a tabular feature representation, 
we unroll the temporal data comprising $D$ features and $L$ timesteps into 
tabular data with $D \cdot L$ features.
For the sake of brevity, and to avoid confusion with temporal features, we 
refer 
to tabular features simply as `timesteps' in the context of evaluation, since 
each tabular feature corresponds to a single 
feature-timestep pair in the original representation.

For TIME, we set the window localization 
parameter $\gamma$ to 0.99 and control the false discovery 
rate at 0.1. We sample 50 permutations to compute each \textit{p}-value.

We generate data sets with 1,000 instances, 10 features and 20 timesteps per 
feature. Five features are randomly selected as relevant. We create a synthetic 
model for each data set, with $\beta$ tuned to yield 
a 90\% accuracy (for 
classification models) or an $R^2$ value of 0.9 (for regression models).
We evaluate the methods by examining power and FDR for identifying relevant 
features as well as timesteps, and average the results over 100 data sets and 
models.

\begin{table}[h]
	\centering
	\caption{Comparison between different explanation methods, indicating 
		average power and FDR for detecting relevant features and timesteps, as 
		well the median runtime for each method.}
	\label{tab:baseline-comparison}
	\resizebox{\linewidth}{!}{%
		\begin{tabular}{|l|ll|ll|r|}
			\hline
			\multicolumn{1}{|c|}{} & 
			\multicolumn{2}{c|}{\begin{tabular}[c]{@{}c@{}} 
					Features\end{tabular}} & 
			\multicolumn{2}{c|}{\begin{tabular}[c]{@{}c@{}} 
					Timesteps\end{tabular}} & \multicolumn{1}{c|}{} \\ 
					\cline{2-5}
			\multicolumn{1}{|c|}{\multirow{-2}{*}{Method}} & 
			\multicolumn{1}{c|}{\begin{tabular}[c]{@{}c@{}} 
					Power\end{tabular}} & 
			\multicolumn{1}{c|}{\begin{tabular}[c]{@{}c@{}} 
					FDR\end{tabular}} & 
			\multicolumn{1}{c|}{\begin{tabular}[c]{@{}c@{}} 
					Power\end{tabular}} & 
			\multicolumn{1}{c|}{\begin{tabular}[c]{@{}c@{}} 
					FDR\end{tabular}} & 
			\multicolumn{1}{c|}{\multirow{-2}{*}{\begin{tabular}[c]{@{}c@{}}
						Runtime\\ (seconds)\end{tabular}}} \\ \hline
			TIME & \textbf{0.930} & 0.037 & \textbf{0.923} & 
			0.054 & 371 \\
			TIME-n & 0.922 & \textbf{0.018} & 0.915 & 
			\textbf{0.021} & 371 \\
			LIME & 0.710 & 0.290 & 0.692 & 0.308 & 682 \\
			SAGE & 0.806 & 0.194 & 0.786 & 0.214 & 15318 \\
			SAGE-m & 0.758 & 0.242 & 0.731 & 0.269 & \textbf{124} \\
			PERM & 0.836 & 0.164 & 0.818 & 0.182 & 1478 \\ \hline
		\end{tabular}%
	}
\end{table}

For the baseline methods, we estimate a feature's 
importance by averaging non-zero importance scores across the timesteps 
belonging to feature.
We sort timesteps in decreasing order of importance scores
and report the $n$ features or timesteps with the highest scores, where $n$ is 
determined by the ground truth 
relevant features and timesteps.
Since our method identifies specific features 
and windows as important, we evaluate it based 
on (i) all the features and timesteps it identifies as 
important, and (ii) up to $n$ timesteps with the 
highest non-zero scores, as we do with the other baselines. We refer to 
these as TIME and TIME-n respectively.

Table~\ref{tab:baseline-comparison} shows results from this comparison, 
averaged across 100 data sets and classification models. Both TIME and TIME-n 
provide higher 
power and lower FDR than all baselines for both features 
and timesteps, and the average FDR is well-controlled at the 0.1 level.

Figure~\ref{fig:heatmaps} illustrates feature importance explanations for one 
such model. It shows a set of heat maps indicating 
relevant timesteps
for the ground truth model along with the importance scores returned by 
TIME and SAGE. For the ground truth model, boxes corresponding to 
relevant timesteps are shown in a uniform color.
For the explanation 
methods, colored boxes indicate non-zero importance scores, with higher scores 
shown in darker shades. Hatched textures are used to show features for which 
ordering is 
relevant (ground truth) or identified as important (TIME), but they are not 
shown for SAGE since it is not able to detect the significance of ordering.
TIME assigns importance scores to windows for each feature, while SAGE 
(as well as other baseline methods) assign importance scores to each timestep, 
since 
they operate on a tabular representation. For this model, TIME identifies 
all the relevant features, timesteps and their ordering correctly. SAGE assigns 
non-zero 
importance scores to all the relevant timesteps, but in some cases, irrelevant 
timesteps are ranked above relevant ones, adversely affecting its power and FDR 
for detecting important features.

\paragraph{Performance vs. test set size.}
In addition to baseline comparisons, we examine the performance of our method 
as a function of the size of the test set used to analyze the model. We 
generate data sets with 6,400 instances, 30 features and 50 timesteps per 
feature, and increase the size of the test set available to the model in 
multiples of two. Ten features are randomly selected as relevant.
For each test set size, we aggregate the results 
over 100 different models.

\begin{figure}[t]
	\centering
	\begin{subfigure}[c]{\linewidth}
		\includegraphics[width=\textwidth]{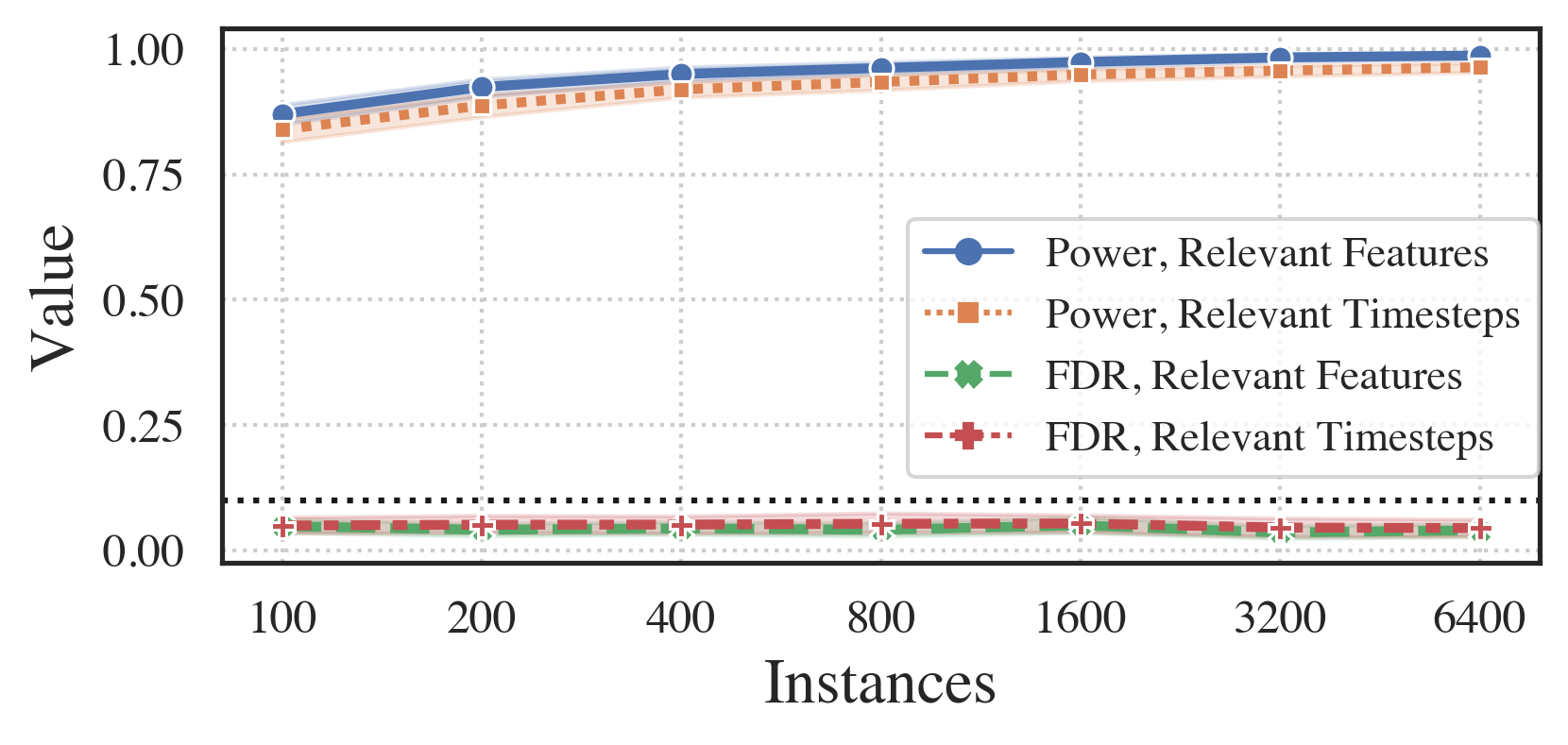}
		\caption{}
		\label{fig:sims-timesteps}
	\end{subfigure} \\
	\begin{subfigure}[c]{\linewidth}
		\includegraphics[width=\textwidth]{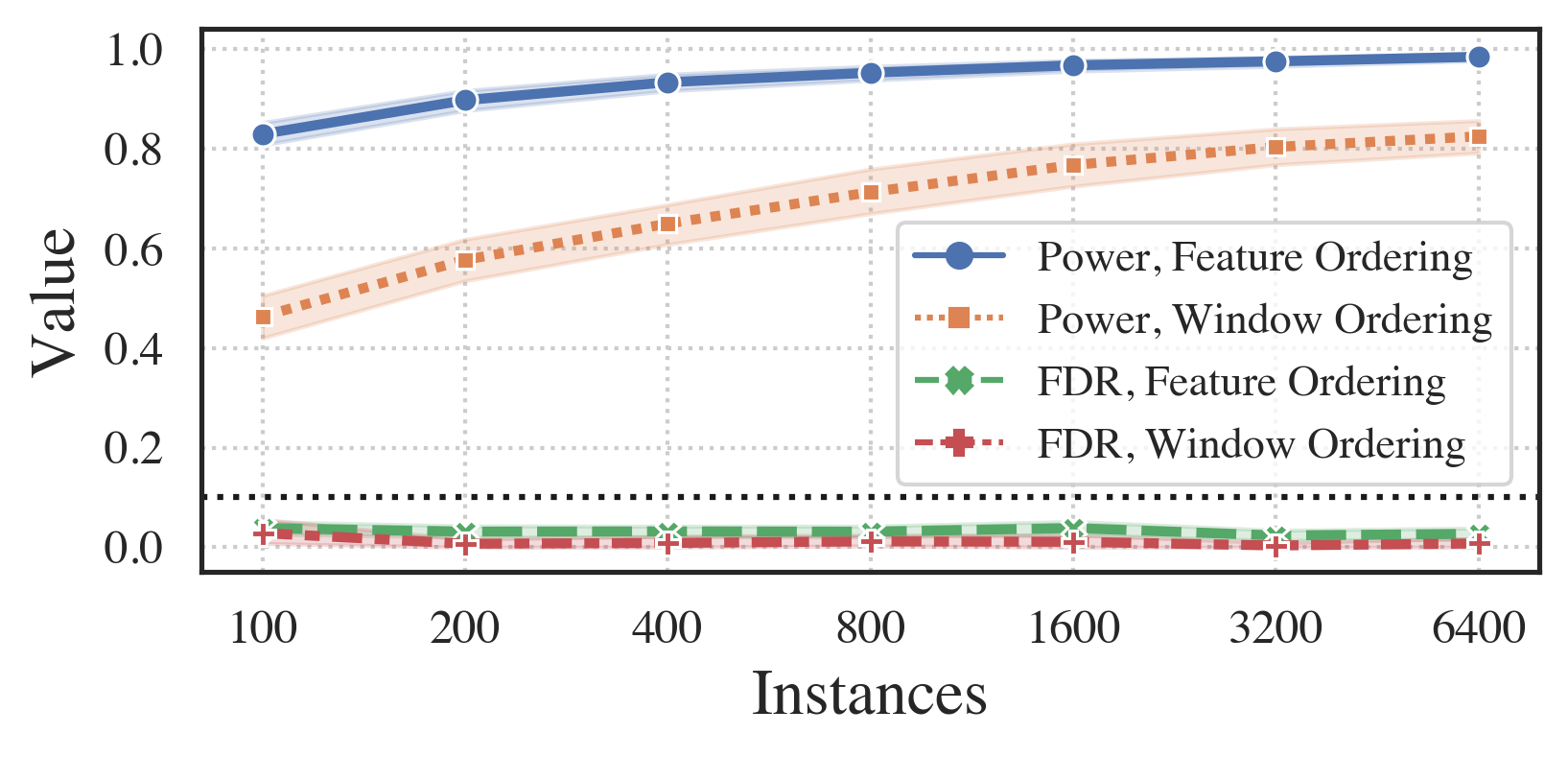}
		\caption{}
		\label{fig:sims-ordering}
	\end{subfigure}
	\caption{(a) Average power and FDR for detecting relevant features and 
		timesteps as a function of test set size. (b) Average power and FDR for 
		detecting ordering relevance for features and windows as a function of 
		test 
		set size. The bands represent 95\% confidence intervals, and the 
		dotted horizontal line represents the 0.1 level at which FDR is 
		controlled.}
	\label{fig:sims-regression}
\end{figure}

Figure~\ref{fig:sims-regression} shows the results of 
this analysis for regression models, and similar results are obtained for 
classification models. Figure~\ref{fig:sims-timesteps} shows average power and 
FDR for relevant features and timesteps as a function of test set size. The 
power increases as the test set size increases and has high terminal values, 
indicating that our approach is successful at identifying most of the relevant 
features and windows. The average FDRs are well-controlled at the 
0.1 level.

Figure~\ref{fig:sims-ordering} shows average power and FDR for 
detecting features and windows for which the ordering of values is important. 
Feature 
ordering refers to the ordering of a feature's values across its entire 
sequence. Since the distribution of values inside the window is different from 
that outside the window, the model is sensitive to the ordering of all features 
having windows smaller than the sequence length. However, the model is 
sensitive to the ordering of values within the window only for certain feature 
functions. At the largest test set size
TIME is able to detect ordering with high accuracy while FDRs 
are well-controlled at the 0.1 level.
We detect window ordering at lower power compared to feature ordering
due to the greater difficulty of the task,
and the fact that relevant features that are 
not identified as important are not assessed for important windows or their 
ordering.

\begin{figure*}[t]
	\centering
	\includegraphics[width=0.8\linewidth]{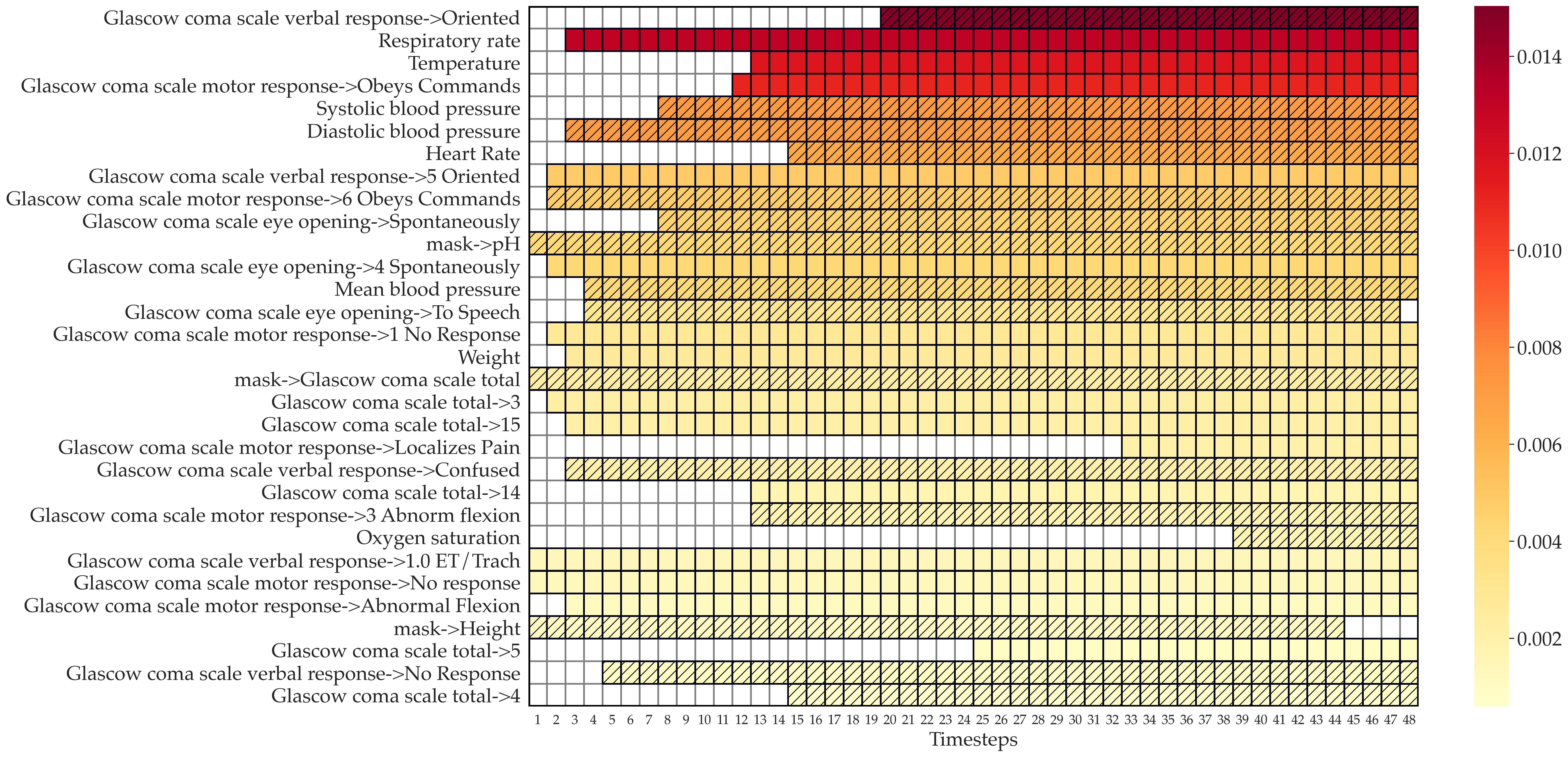}
	\caption{Heat map showing the analysis of a MIMIC-III LSTM model trained 
		to predict in-hospital mortality. Out of a
		total of 76 features, 31 were identified as important and are 
		shown in decreasing order of their importance score.
		Each row corresponds to a single feature 
		and shows the window corresponding to important timesteps in color. The 
		importance score is indicated by the color bar, with higher scores 
		associated with darker shades, and hatched textures show windows that 
		were found to be significant in relation to ordering.}
	\label{fig:mimic}
\end{figure*}

\subsection{MIMIC-III Benchmark LSTM Model}

To consider a challenging, real-world task, we analyze an LSTM model trained on 
the MIMIC-III database, a 
publicly available critical care database consisting of records of 58,976 
intensive care unit (ICU) 
admissions~\cite{johnsonMIMICIIIFreelyAccessible2016}. The model is one of 
several proposed as part of a benchmark suite for four different clinical 
prediction tasks over the MIMIC-III database 
\cite{harutyunyanMultitaskLearningBenchmarking2017}. In this analysis, we 
examine a regular LSTM~\cite{hochreiterLongShortTermMemory1997} trained to 
predict in-hospital mortality of patients given the first 48 hours of their
ICU stay observations. After pruning to meet inclusion criteria (such as stays 
of 48 hours or longer), the data comprises training, validation and test sets 
of 14,682, 3,221 and 3,236 stays respectively, with 13.23\% of the labels being 
positive. There are 76 features, each 
represented by a sequence of length 48. The features are derived from chart and 
laboratory measurements, and include `mask' features indicating interpolated 
values. Further details on the model and features may be found in the 
benchmarking paper \cite{harutyunyanMultitaskLearningBenchmarking2017}.

We use the validation set to analyze the LSTM and 
identify important features and windows, and whether or not their ordering is 
important to the model.
We set the window localization parameter $\gamma$ as 0.95 and control the false 
discovery rate at 0.1. We sample 200 permutations to compute each 
\textit{p}-value.
Figure~\ref{fig:mimic} shows the results of this analysis. The feature 
importance analysis of the MIMIC-III model identifies a set of 31 features that 
are important for the model's predictions, as well as the 
important windows for these features. The important windows almost always focus 
on the more recent part of the patients' histories, which is expected since 
death is more likely to be indicated by abnormalities in the 
later stages of the ICU stay. We also note 
that the ordering of timesteps is found to be important for some features, 
suggesting that 
the model may be picking up on trends for these features.

To validate our analysis, we use the set of features and windows estimated to 
be important to perform feature selection. We prune the features that are not 
estimated to be important 
and set the out-of-window timesteps for important features to zero. We then 
retrain the LSTM on the pruned data set and compare its performance to the 
original model on the 
held-aside test set. We further train and test 20 feature-selected models with 
31 features and windows chosen at random. The area under the ROC curve for the 
retrained 
model on the test set is 0.836, which is close to that of the original model 
(0.838) but significantly higher than the models using randomly selected 
features (mean 0.801, SD 0.015), suggesting that TIME is able to 
identify a salient subset of features and windows for this model.


\section{Conclusions}

We have presented TIME, a method to explain black-box models having an 
explicit sequential or temporal 
structure. TIME identifies the set of important features and their 
degree of importance, and for each important feature, it identifies the window 
that the model focuses on and the significance of its ordering. It uses 
hypothesis testing and an FDR control methodology to detect 
these with statistical rigor.

Our experiments show that on synthetic data, TIME performs significantly better 
than baseline methods at identifying relevant features and timesteps.  
Additionally, in comparison to baselines, 
TIME is efficient to compute and potentially more interpretable since it 
identifies contiguous windows rather than scattered timesteps.
We apply 
TIME to an LSTM trained to predict risk of in-hospital mortality from ICU data, 
and we identify salient features, windows and ordering in patients' clinical 
histories that the model focuses on. We show that a model trained using 
features and timesteps selected using this analysis performs nearly as well as 
the original model.

There are several limitations of this work that we plan to address in future 
research. Our approach for permutation across 
sequences currently assumes regularly sampled, time-aligned and fixed-length 
sequences, 
which excludes models trained over irregularly sampled, variable-length 
sequences.
We assume that there exists a single contiguous window that is important, which 
may not be the case. We plan to explore the use of conditional permutation 
tests~\cite{stroblConditionalVariableImportance2008} to avoid breaking 
across-feature 
correlations, though efficient and accurate estimation of conditional 
distributions for temporal data remains a challenging task.
We also plan to extend our approach to analyze sequence-sequence models,
and to consider windows that are aligned in other ways, such as on
an absolute scale (e.g., dates on the Gregorian calendar) or a relative scale 
(e.g., patient age).



\clearpage
\bibliography{refs}
\bibliographystyle{icml2021}
\newpage
\setcounter{table}{0}
\setcounter{figure}{0}
\setcounter{equation}{0}
\setcounter{section}{0}

\renewcommand{\thesection}{S\arabic{section}}
\renewcommand{\thetable}{S\arabic{table}}
\renewcommand{\thefigure}{S\arabic{figure}}
\renewcommand{\theequation}{S\arabic{equation}}
\newcommand*\sref[1]{%
	S\ref{#1}}
\newcommand*\sfref[1]{%
	Supplementary Figure S\ref{#1}}
\newcommand*\stref[1]{%
	Supplementary Table S\ref{#1}}
\newcommand*\smref[1]{%
	Supplementary Materials S\ref{#1}}

\newcommand{\Supp}[1]{\textcolor{Green}{}}
\newcommand{\TODO}[1]{\textcolor{Orange}{[#1]}}
\newcommand{\Comment}[1]{\textcolor{Red}{[#1]}}

\section{Importance Scores}

Equation~\ref{eqn:importance-tensor}, reiterated 
below, computes 
the importance score for window $[k_1, k_2]$ of feature $j$:
\begin{equation*}
	I \left( f, j, [k_1, k_2] \right)
	= \frac{1}{M} \sum\limits_{i=1}^{M} \Delta \bar{\mathcal{L}}_{j, [k_1, 
		k_2]}^{(i)}.
	\label{eqn:supp-importance-tensor}
\end{equation*}
Intuitively, this captures the notion that a window for a given feature is 
considered important to the model if it has a positive association with the 
target, in the 
sense that permuting its values increases the model's loss on average. In 
Proposition~\ref{prop:supp-importance-tensor}, we show that for additive models 
and under certain 
assumptions, Equation~\ref{eqn:importance-tensor} captures this 
association in terms of the covariance between the target and the 
feature function operating on the window.

\begin{prop}
	Let $\mathbf{X}_1 \ldots \mathbf{X}_D$ be independent random vectors of 
	size $L$ representing
	temporal features $\mathcal{F} = \{1, 2, \ldots, D\}$ for the additive 
	model $f(\mathbf{X}_1, \ldots 
	\mathbf{X}_D) = 
	\sum\limits_{j' \in \mathcal{F}} g_{j'}(\mathbf{X}_{j'})$.
	Let $\W$ be a window for $\bX$ perturbed according to
	Equation~\ref{eqn:perturbed-feature-tensor}, and let $\Wbar$ represent 
	the timesteps outside the window. Let
	$\g(\bX)$ decompose additively over the sequence as:
	\[ \g(\bX) = g_{j, 
	\W}(\mathbf{X}_{j, \W}) + g_{j, \Wbar}(\mathbf{X}_{j, \Wbar}) \]
	where 
	$\mathbf{X}_{j, \W}$ and $\mathbf{X}_{j, \Wbar}$ represent subsequences of 
	$\bX$ inside and outside the window respectively, and $g_{j, 
	\W}(\mathbf{X}_{j, \W})$ and $g_{j, \Wbar}(\mathbf{X}_{j, \Wbar})$ 
	represent functions over these subsequences.
	Let $\Z$ be the target and
	$\mathcal{L}$ be the quadratic loss function.

	Then, the importance score
	$I(f, \bX, \W)$ of the window $\W$ for feature $\bX$ as computed by
	Equation~\ref{eqn:importance-tensor} satisfies:
	\begin{equation}
	\begin{split}
	\E \left[ I(f, \bX, \W) \right]
	\;=\;2 \Bigg[ &\cov \bigg( \Z, g_{j, \W}(\mathbf{X}_{j, \W}) \bigg) \\
	-\;&\cov
	\bigg( g_{j, \W}(\mathbf{X}_{j, \W}), g_{j, \Wbar}(\mathbf{X}_{j, \Wbar}) 
	\bigg) \Bigg].
	\end{split}
	\end{equation}
	\label{prop:supp-importance-tensor}
\end{prop}

\begin{proof}
	Our formulation of the importance score corresponds to an empirical 
	estimate of difference-based model reliance~\cite{fisherAllModelsAre2019}, 
	and equivalently, when quadratic loss is assumed, the permutation 
	importance score 
	in~\citet{gregoruttiCorrelationVariableImportance2017}, extended to the 
	temporal setting. In the tabular 
	setting, when the model is assumed to be 
	additive, the expected overall importance score of feature 
	$\bX$ is given by~\cite{fisherAllModelsAre2019}:
\begin{align*}
\E \left[ I(f, \bX) \right] \;=\; 2 \Bigg[ &\cov \bigg( Y, \g(\bX) \bigg) \\
-\;&\sum\limits_{\bar{j} \in \mathcal{F}
\backslash j} \cov \bigg( \g(\bX), g_{\bar{j}}(\mathbf{X}_{\bar{j}}) \bigg) 
\Bigg]
\end{align*}
where $\mathbf{X}_{\bar{j}}$ corresponds to the subset of features other than 
$j$. Extending this to temporal models, the expected 
importance score for window $\W$ of feature $\bX$ is given by:
\begin{align*}
\E \left[ I(f, \bX, \W) \right] \;=\; 2 \Bigg[ &\cov \bigg( Y, g_{j, 
\W}(\mathbf{X}_{j, \W}) \bigg) \\
-\;&\sum\limits_{\bar{j} \in \mathcal{F}
	\backslash j} \cov \bigg( g_{j, \W}(\mathbf{X}_{j, \W}), 
	g_{\bar{j}}(\mathbf{X}_{\bar{j}}) \bigg) \; \\
-\;&\cov \bigg( g_{j, \W}(\mathbf{X}_{j, \W}), g_{j, \Wbar}(\mathbf{X}_{j, 
\Wbar}) 
\bigg) \Bigg] \\
\;=\;2 \Bigg[ &\cov \bigg( \Z, g_{j, \W}(\mathbf{X}_{j, \W}) \bigg) \\
-\;&\cov
\bigg( g_{j, \W}(\mathbf{X}_{j, \W}), g_{j, \Wbar}(\mathbf{X}_{j, \Wbar}) 
\bigg) \Bigg]
\end{align*}
since $\cov \bigg( g_{j, \W}(\mathbf{X}_{j, \W}), 
g_{j'}(\mathbf{X}_{\bar{j}}) \bigg) = 0\;\forall\;\bar{j} \in \mathcal{F} 
\backslash j$ as $\mathbf{X}_1, \ldots \mathbf{X}_D$ are independent.
\end{proof}

\begin{cor}
	Let $\Wstar$ represent the relevant window, so that $\g(\bX) =
	g_{j, \Wstar}(\mathbf{X}_{j, \Wstar})$ and $g_{j, \Wstarbar}(\mathbf{X}_{j, 
	\Wstarbar}) = 
	0$. Then:
	\begin{equation}
	\begin{split}
	\E \left[ I(f, \bX, [1, L]) \right] &= \E \left[ I(f, \bX, \Wstar) \right]\\
	&= 2\;\cov \left( \Z, \g(\bX) \right)
	\end{split}
	\label{eqn:ground-truth-importance-score}
	\end{equation}
\end{cor}

\section{Synthetic Data Generation}
\label{section:supp-synthetic-data}

We use Markov chains to generate time series 
data, as shown in Figure~\ref{fig:sim-generator}.
Each feature is associated with a randomly selected window and a pair of Markov 
chains, one each to generate values for in-window and out-of-window timesteps.
The number of states in each chain is sampled uniformly at random between 2 and 
5.
The features include a combination of continuous and discrete features.
Each state $m$ is associated with a Gaussian 
random variable $S_m \sim \mathcal{N}(\mu_m, \sigma_m^2)$ (for continuous 
features) or an integer value (for discrete features) and transition 
probabilities $p_{mn}$ to 
other states $n$ within the same chain. The mean, standard deviation, and 
transition probabilities for each state are sampled uniformly at random.
The sequence for a given instance and feature is generated 
via a random walk through the chains. For example, a sequence $i$ for 
feature $j$ with 5 timesteps and underlying window [2, 3] may be generated 
as: 
$\mathbf{x}_j^{(i)} = \left< x_{j, 
	1}^{(i)}, x_{j, 2}^{(i)}, x_{j, 3}^{(i)}, x_{j, 4}^{(i)}, x_{j, 5}^{(i)} 
\right> = \left< s'_{0, 1}, s'_{2, 2}, s_{1, 3}, s_{1, 4}, s'_{0, 5} 
\right>$, where $s_{m, t}$ is sampled from $S_m$ at timestep $t$. 
For some continuous 
features, sampled values are aggregated over time to model 
increasing, constant, or decreasing trends, as shown in 
Figure~\ref{fig:sim-series}. In this case, for each timestep $t$, 
$x_{j, t} = s_{m, t} + \sum\limits_{k=1}^{t-1} x_{j, k}$.
\begin{figure}[h]
	\centering
	\begin{subfigure}[c]{\linewidth}
		\includegraphics[width=\textwidth]{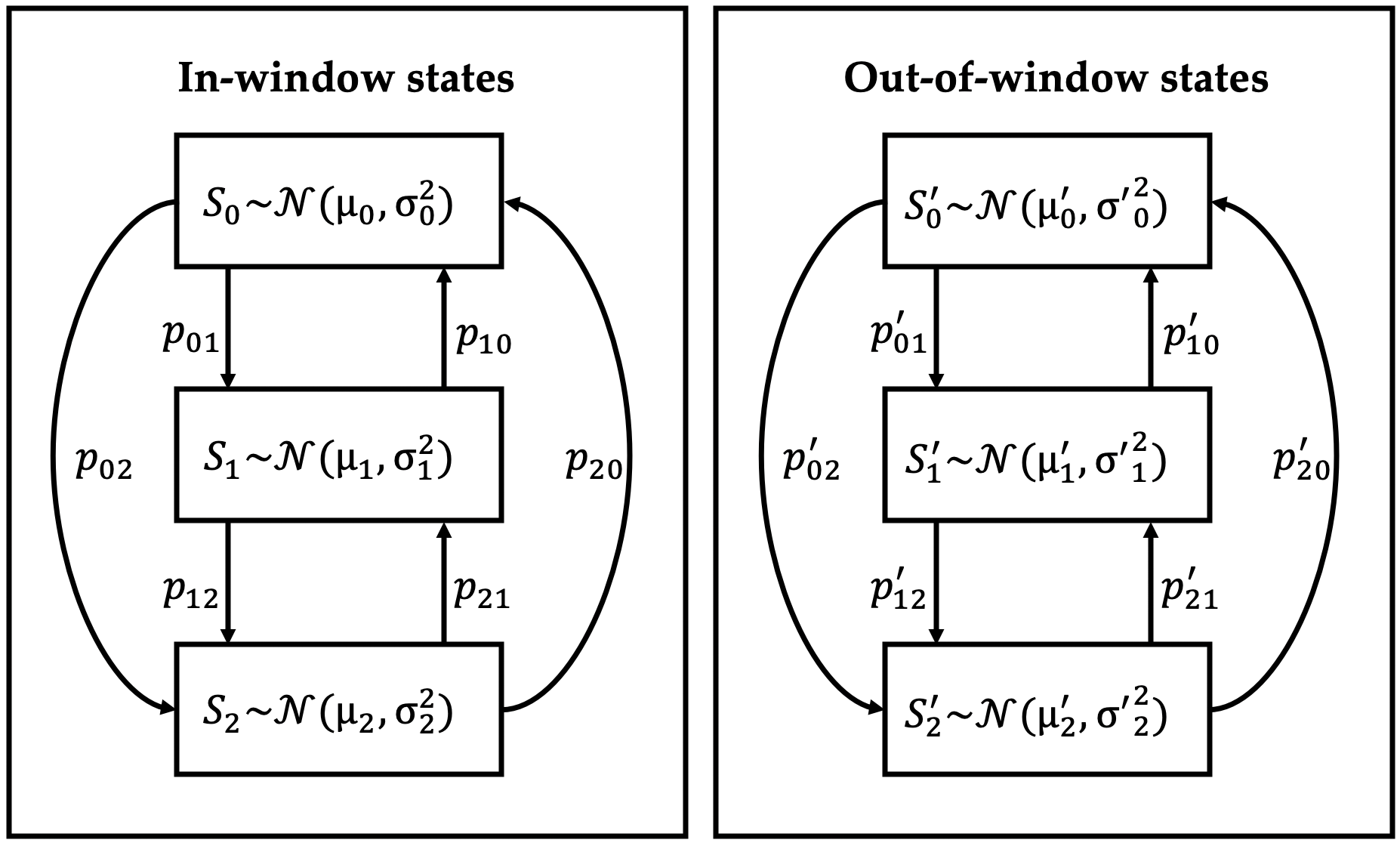}
		\caption{}
		\label{fig:sim-generator}
	\end{subfigure} \hfill
	\begin{subfigure}[c]{\linewidth}
		\includegraphics[width=\textwidth]{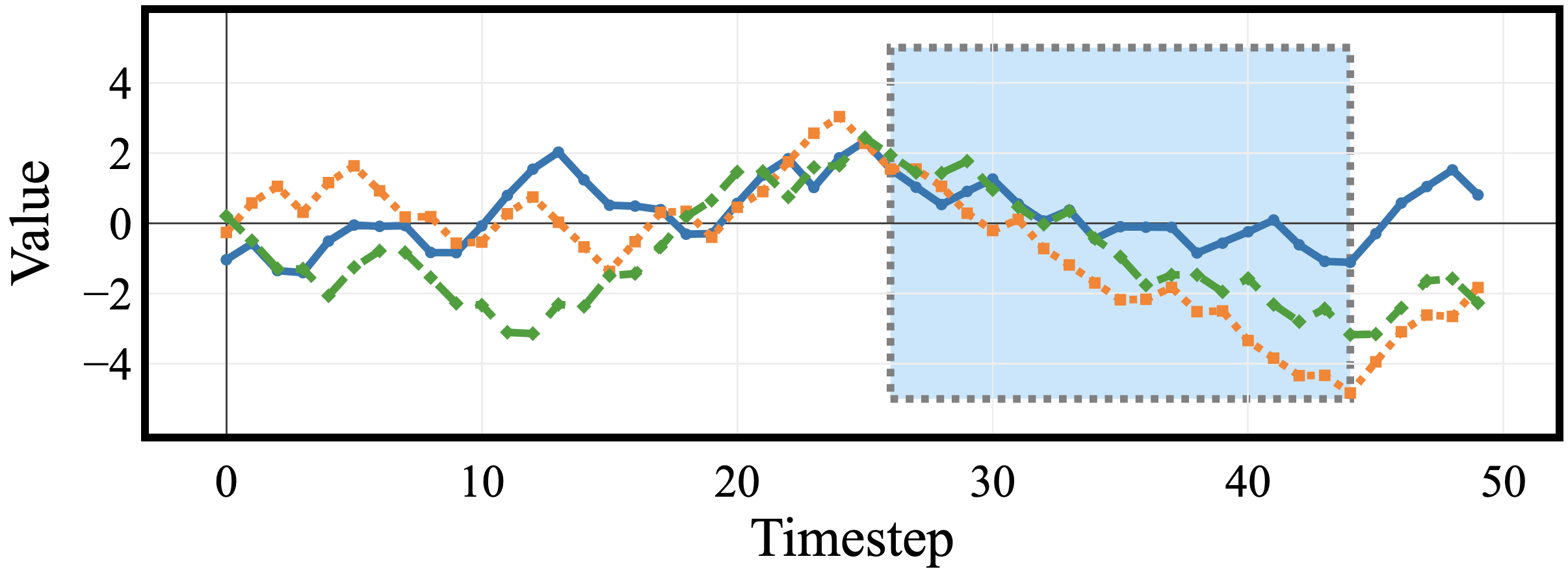}
		\caption{}
		\label{fig:sim-series}
	\end{subfigure}
	\caption{(a) Generator for a continuous feature consisting of two Markov 
		chains, one each for in-window and out-of-window states. Here each 
		Markov 
		chain consists of three states, and each state is associated with a 
		Gaussian random variable.
		(b) Three sequences generated via random walks through the chains, 
		with  
		the sampled values aggregated over time to create trends. The window is 
		represented by blue shading.}
	\label{fig:sim}
\end{figure}

\section{Empirical Results}
\paragraph{Additional baseline comparisons.}
In addition to the comparisons shown in 
Table~\ref{tab:baseline-comparison}, we also tested variants of two 
baseline methods: (i) SAGE-z, an approximation of SAGE provided by its authors
where held-out features are sampled from all-zero reference values, instead of 
mean values as in SAGE-m, and (ii) PERM-f, a variant of PERM where FDR is 
controlled at the 0.1 level (as is done for TIME) using the BH 
procedure~\cite{benjaminiControllingFalseDiscovery1995} 
after computing importance scores for each timestep. These results are 
shown in 
Table~\ref{tab:supp-baseline-comparison} along with the results from 
Table~\ref{tab:baseline-comparison}.

We also performed baseline comparisons using a larger feature set composed of 
30 features, 10 out of which were randomly selected as relevant. 
Table~\ref{tab:supp-baseline-comparison-30} shows these results, aggregated 
over 100 different models. In case of SAGE and SAGE-m, these results are 
aggregated over 86 and 99 models respectively, as convergence was not achieved 
for some models 
within the designated time limit (72 hours per model). PERM-f is not included 
since it did not identify any features as important. The results corroborate 
the conclusions drawn from Table~\ref{tab:baseline-comparison}.

\begin{table}[h]
	\centering
	\caption{Comparison between different explanation methods, indicating 
		average power and FDR for detecting relevant features and timesteps, as 
		well the median runtime for each method.}
	\label{tab:supp-baseline-comparison}
	\resizebox{\linewidth}{!}{%
		\begin{tabular}{|l|ll|ll|r|}
			\hline
			\multicolumn{1}{|c|}{} & 
			\multicolumn{2}{c|}{\begin{tabular}[c]{@{}c@{}} 
					Features\end{tabular}} & 
			\multicolumn{2}{c|}{\begin{tabular}[c]{@{}c@{}} 
					Timesteps\end{tabular}} & \multicolumn{1}{c|}{} \\ 
			\cline{2-5}
			\multicolumn{1}{|c|}{\multirow{-2}{*}{Method}} & 
			\multicolumn{1}{c|}{\begin{tabular}[c]{@{}c@{}} 
					Power\end{tabular}} & 
			\multicolumn{1}{c|}{\begin{tabular}[c]{@{}c@{}} 
					FDR\end{tabular}} & 
			\multicolumn{1}{c|}{\begin{tabular}[c]{@{}c@{}} 
					Power\end{tabular}} & 
			\multicolumn{1}{c|}{\begin{tabular}[c]{@{}c@{}} 
					FDR\end{tabular}} & 
			\multicolumn{1}{c|}{\multirow{-2}{*}{\begin{tabular}[c]{@{}c@{}}
						Runtime\\ (seconds)\end{tabular}}} \\ \hline
			TIME & \textbf{0.930} & 0.037 & \textbf{0.923} & 
			0.054 & 371 \\
			TIME-n & 0.922 & \textbf{0.018} & 0.915 & 0.021 & 371 \\
			LIME & 0.710 & 0.290 & 0.692 & 0.308 & 682 \\
			SAGE & 0.806 & 0.194 & 0.786 & 0.214 & 15318 \\
			SAGE-m & 0.758 & 0.242 & 0.731 & 0.269 & 124 \\
			SAGE-z & 0.656 & 0.344 & 0.648 & 0.352 & 
			\textbf{44} \\
			PERM & 0.836 & 0.164 & 0.818 & 0.182 & 1478 \\
			PERM-f & 0.326 & 0.024 & 0.312 & 
			\textbf{0.008} & 1478 
			\\ \hline
		\end{tabular}%
	}
\end{table}

\begin{table}[h]
	\centering
	\caption{Comparison between different explanation methods for models 
	composed of 30 features, indicating 
		average power and FDR for detecting relevant features and timesteps, as 
		well the median runtime for each method.}
	\label{tab:supp-baseline-comparison-30}
	\resizebox{\linewidth}{!}{%
		\begin{tabular}{|l|ll|ll|r|}
			\hline
			\multicolumn{1}{|c|}{} & 
			\multicolumn{2}{c|}{\begin{tabular}[c]{@{}c@{}} 
					Features\end{tabular}} & 
			\multicolumn{2}{c|}{\begin{tabular}[c]{@{}c@{}} 
					Timesteps\end{tabular}} & \multicolumn{1}{c|}{} \\ 
			\cline{2-5}
			\multicolumn{1}{|c|}{\multirow{-2}{*}{Method}} & 
			\multicolumn{1}{c|}{\begin{tabular}[c]{@{}c@{}} 
					Power\end{tabular}} & 
			\multicolumn{1}{c|}{\begin{tabular}[c]{@{}c@{}} 
					FDR\end{tabular}} & 
			\multicolumn{1}{c|}{\begin{tabular}[c]{@{}c@{}} 
					Power\end{tabular}} & 
			\multicolumn{1}{c|}{\begin{tabular}[c]{@{}c@{}} 
					FDR\end{tabular}} & 
			\multicolumn{1}{c|}{\multirow{-2}{*}{\begin{tabular}[c]{@{}c@{}}
						Runtime\\ (seconds)\end{tabular}}} \\ \hline
			TIME & \textbf{0.914} & 0.033 & \textbf{0.909} & 0.037 & 2810
\\
			TIME-n & 0.909 & \textbf{0.015} & 0.904 & \textbf{0.011} & 2810
\\
			LIME & 0.728 & 0.272 & 0.680 & 0.320 & 1704
\\
			SAGE & 0.799 & 0.201 & 0.743 & 0.257 & 177038
\\
			SAGE-m & 0.692 & 0.308 & 0.642 & 0.358 & 2463
\\
			SAGE-z & 0.609 & 0.391 & 0.583 & 0.417 & \textbf{667}
\\
			PERM & 0.830 & 0.170 & 0.792 & 0.208 & 11365
\\ \hline
		\end{tabular}%
	}
\end{table}

\paragraph{Performance vs. test set size for classification models.} 
Figure~\ref{fig:supp-sims-classification} shows results for synthetic 
classification 
models analogous to those for synthetic 
regression models shown in Figure~\ref{fig:sims-regression}. The results 
show the same patterns as for synthetic
regression models: the power for detecting relevant features, 
timesteps and ordering increases as the size of the test set used to examine 
the model is increased, and the FDR is well-controlled at the 0.1 level.

\begin{figure}[t]
	\centering
	\begin{subfigure}[c]{\linewidth}
		\includegraphics[width=\textwidth]{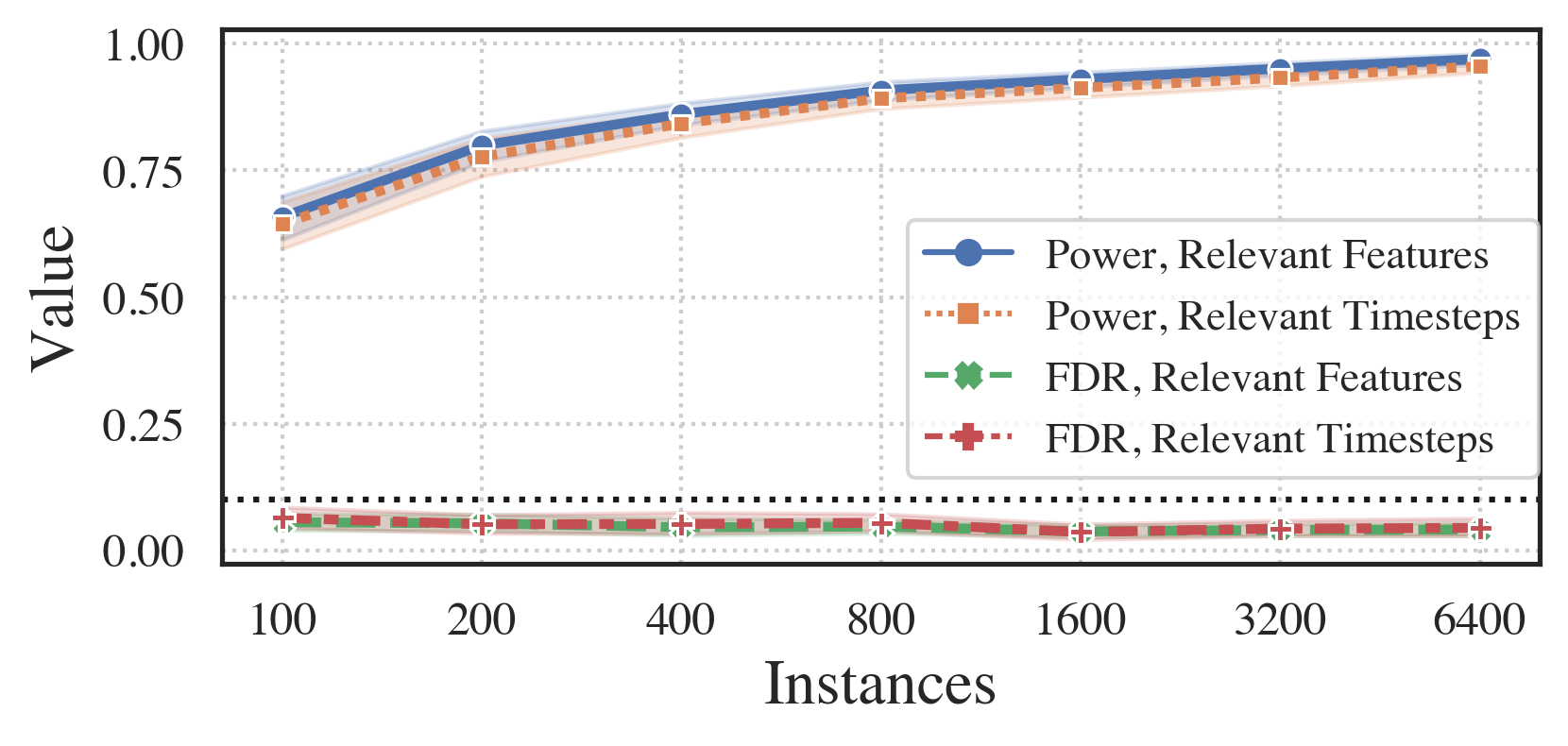}
		\caption{}
		\label{fig:supp-sims-timesteps-classification}
	\end{subfigure} \\
	\begin{subfigure}[c]{\linewidth}
		\includegraphics[width=\textwidth]{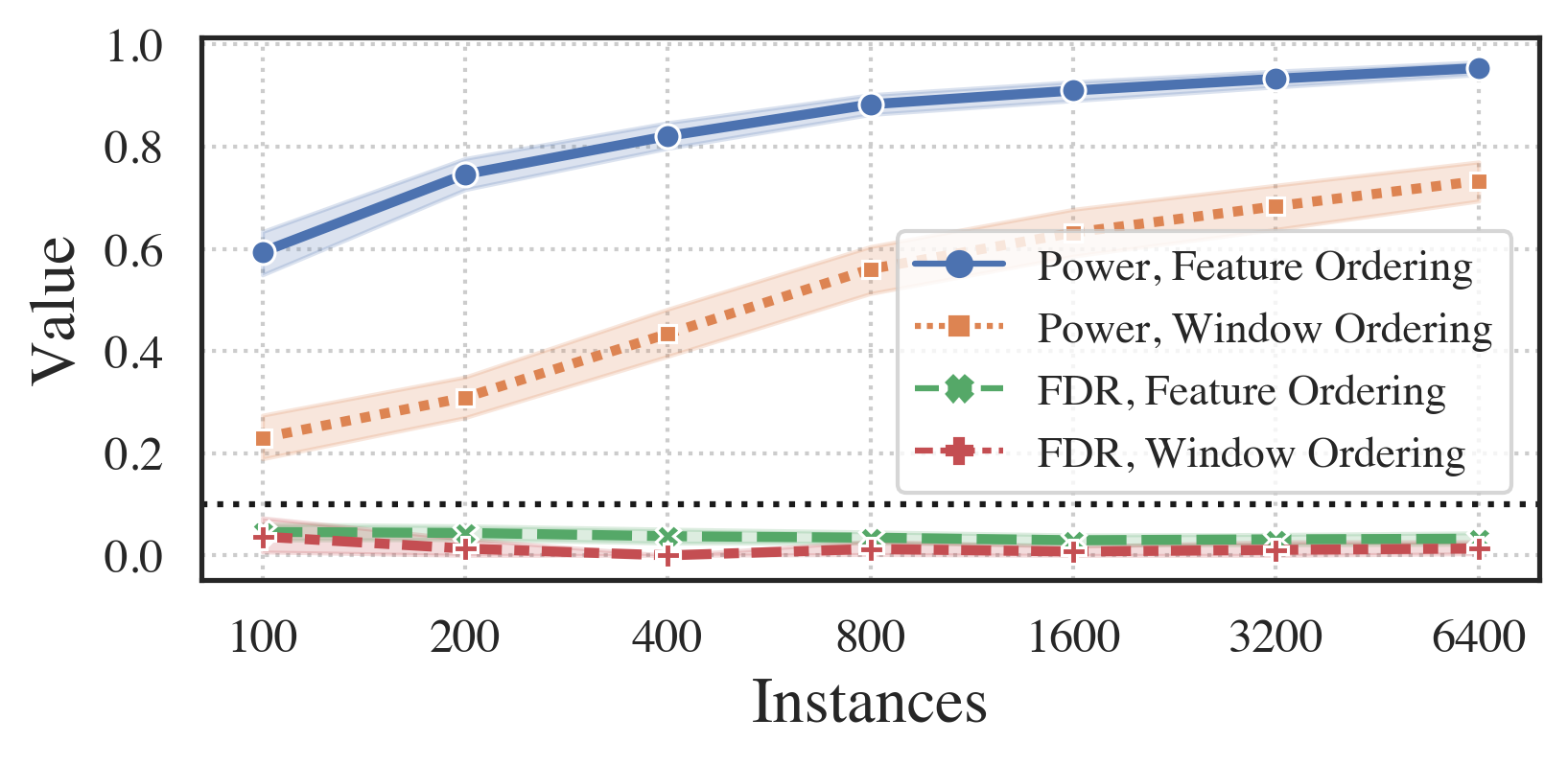}
		\caption{}
		\label{fig:supp-sims-ordering-classification}
	\end{subfigure}
	\caption{(a) Average power and FDR for detecting relevant features and 
		timesteps as a function of test set size for classification models. 
		(b) 
		Average power and FDR for 
		detecting ordering relevance for features and windows as a function of 
		test set size. The bands represent 95\% confidence intervals, and the 
		dotted horizontal line represents the 0.1 level at which FDR is 
		controlled.}
	\label{fig:supp-sims-classification}
\end{figure}

\section{Computation}

\paragraph{Time complexity.}
The time complexity of our method is 
$\mathcal{O}(M P L \min\{R \log{L}, D\})$, where $M$, $D$, 
$L$, $R$ and $P$ are the 
test-set size, number of features, sequence length, number of relevant 
features and number of permutations respectively. The 
logarithmic term arises since the window search 
algorithm partitions the search space in half at each step of the window 
search. In practice, as shown in Table~\ref{tab:baseline-comparison}, 
TIME is often significantly faster than PERM, since the permutation of 
high-value subsets (windows) can leverage a vectorized implementation to yield 
faster performance than permutations of their constituent timesteps.

\paragraph{Distributed implementation.}
TIME can be computed significantly faster by leveraging a distributed computing 
environment, 
where each node analyzes a feature or a subset of features. The results shown 
in Figure~\ref{fig:sims-regression} and Figure~\ref{fig:mimic} were 
produced with a distributed implementation using
\href{https://research.cs.wisc.edu/htcondor/index.html}{HTCondor}. For 
a fair comparison, the distributed implementation 
was disabled while comparing running times of different baseline methods shown 
in Table~\ref{tab:baseline-comparison}.

The code implementing our algorithm and simulations is available 
at 
\href{https://github.com/Craven-Biostat-Lab/anamod}{https://github.com/Craven-Biostat-Lab/anamod}.
 The code for training the MIMIC-III model is based 
on~\citet{harutyunyanMultitaskLearningBenchmarking2017}, available at
\href{https://github.com/YerevaNN/mimic3-benchmarks}{https://github.com/YerevaNN/mimic3-benchmarks}.
 The MIMIC-III data 
set~\cite{johnsonMIMICIIIFreelyAccessible2016} is 
available at \href{https://mimic.physionet.org/}{https://mimic.physionet.org/}.

\end{document}